\documentclass[lettersize,journal]{IEEEtran}
\usepackage{amsmath,amsfonts}
\usepackage{algorithmic}
\usepackage{algorithm}
\usepackage{array}
\usepackage[caption=false,font=normalsize,labelfont=sf,textfont=sf]{subfig}
\usepackage{textcomp}
\usepackage{stfloats}
\usepackage{url}
\usepackage{verbatim}
\usepackage{graphicx}
\usepackage{cite}

\usepackage{booktabs}       
\usepackage{amsfonts}       
\usepackage{nicefrac}       
\usepackage{microtype}      
\usepackage{xcolor}         
\usepackage{amsmath}
\usepackage{amssymb}
\usepackage{mathtools}
\usepackage{amsthm}
\newtheorem{definition}{Definition}
\newtheorem{lemma}{Lemma}
\newtheorem{theorem}{Theorem}

\usepackage{multirow}
\usepackage[table]{xcolor}
\usepackage{color}
\definecolor{mygreen}{RGB}{146, 199, 113} 
\usepackage{graphicx}
\usepackage{adjustbox}
\usepackage{array}
\usepackage{booktabs}
\usepackage{dsfont}
\usepackage{makecell}
\usepackage{wrapfig}
\usepackage{hyperref}
\usepackage{mathrsfs}
\usepackage{xspace}
\usepackage{pifont}

\def \waveicon {\includegraphics[width=1.6ex]{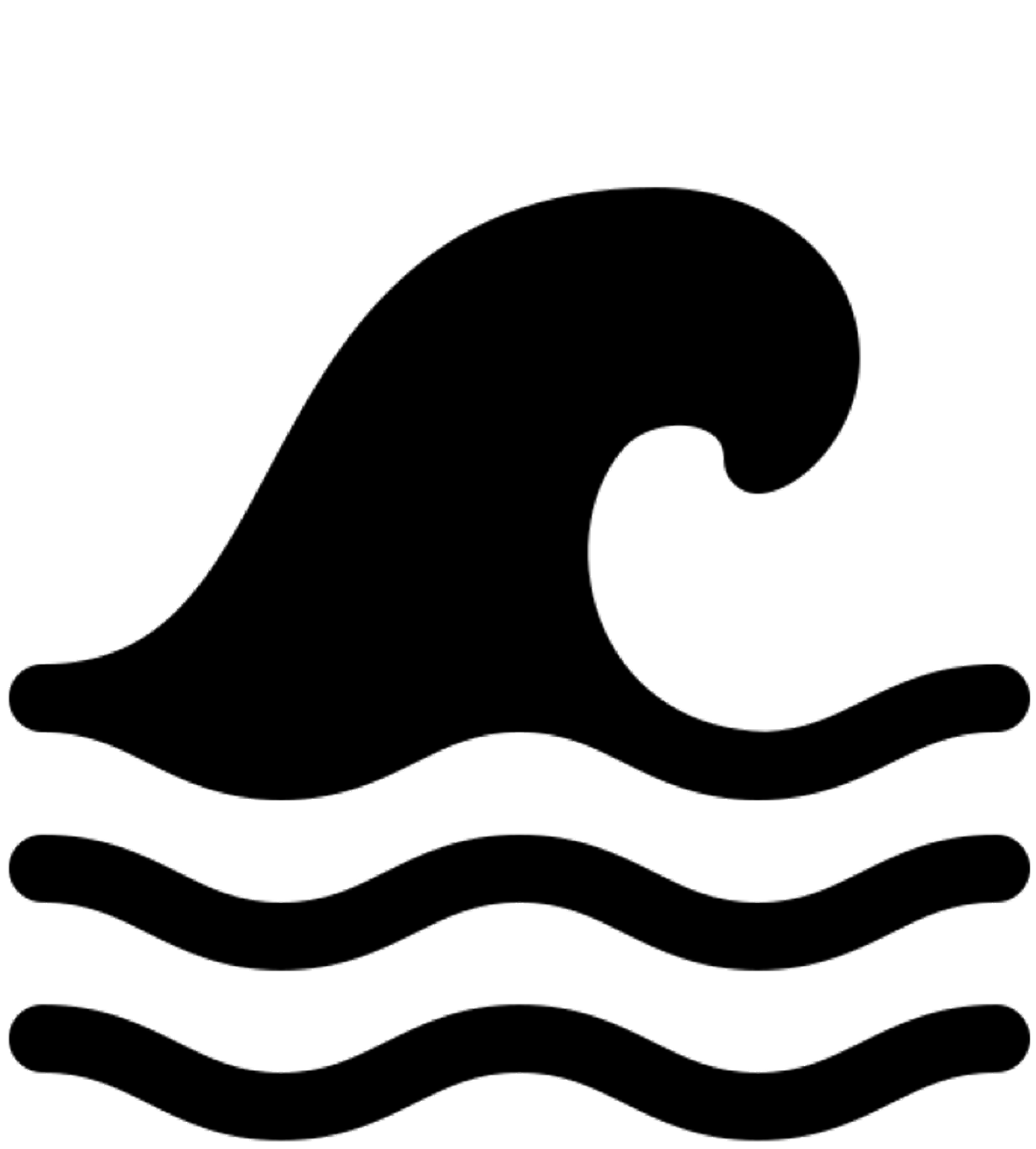}\xspace}

\begin{document}

\title{Constraint-based Pre-training: From Structured Constraints to Scalable Model Initialization}

\author{Fu Feng,\; Yucheng Xie,\; Ruixiao Shi,\; Jing Wang$^*$,\; Xin Geng$^*$,~\IEEEmembership{Senior Member,~IEEE}
\thanks{F. Feng, Y. Xie, R. Shi, J. Wang and X. Geng are with the School of Computer Science and Engineering, Southeast University, Nanjing, China and the Key Laboratory of New Generation Artificial Intelligence Technology and Its Interdisciplinary Applications (Southeast University), Ministry of Education, China 
        (E-mail: \{fufeng, xieyc, eric\_xiao, wangjing91, xgeng\}@seu.edu.cn).}
\thanks{$^*$ Corresponding authors}
}

\markboth{Journal of \LaTeX\ Class Files,~Vol.~14, No.~8, August~2021}%
{Shell \MakeLowercase{\textit{et al.}}: A Sample Article Using IEEEtran.cls for IEEE Journals}


\maketitle

\begin{abstract}
The pre-training and fine-tuning paradigm has become the dominant approach for model adaptation. 
However, conventional pre-training typically yields models at a fixed scale, whereas practical deployment often requires models of varying sizes, exposing its limitations when target model scales differ from those used during pre-training.
To address this, we propose an innovative \textit{constraint-based pre-training paradigm} that imposes structured constraints during pre-training to disentangle size-agnostic knowledge into \textit{reusable weight
templates}, while assigning size-specific adaptation to \textit{lightweight weight scalers}, thereby reformulating variable-sized model initialization as a multi-task adaptation problem.
Within this paradigm, we further introduce WeiT, which employs Kronecker-based constraints to regularize the pre-training process. Specifically, model parameters are represented as compositions of weight templates via concatenation and weighted aggregation, with adaptive connections governed by lightweight weight scalers whose parameters are learned from limited data. This design enables flexible and efficient construction of model weights across diverse downstream scales.
Extensive experiments demonstrate the efficiency and effectiveness of WeiT, achieving state-of-the-art performance in initializing models with varying depths and widths across a broad range of perception and embodied learning tasks, including  \textsc{Image Classification}, \textsc{Image Generation}, and \textsc{Embodied Control}. 
Moreover, its effectiveness generalizes to both Transformer-based and Convolution-based architectures, consistently enabling faster convergence and improved performance even under full training.
\end{abstract} 

\section{Introduction}
\IEEEPARstart{F}{ine-tuning} pre-trained models has emerged as the dominant paradigm for adapting foundation models to downstream tasks~\cite{ding2023parameter, han2022survey, awais2025foundation}, particularly in data-scarce scenarios, where training modern architectures such as Vision Transformers (ViTs)~\cite{dosovitskiy2020image} \textit{from scratch} is often impractical due to their reliance on large-scale data and substantial computational cost.

However, in practical deployment, models are often \textit{subject to constraints} such as memory usage~\cite{han2015deep}, computational resources~\cite{tan2019efficientnet}, and response time~\cite{zhang2022minivit}, necessitating models of \textit{variable sizes}. 
In contrast, most off-the-shelf pre-trained models are typically available only in a limited set of fixed configurations (e.g., ViT-B with 12 layers~\cite{touvron2021training}). 
Consequently, target model sizes that \textit{fall outside these predefined configurations} still require re-pre-training on large-scale datasets prior to deployment, incurring substantial training costs.

\begin{figure}[t]
  \centering
  \includegraphics[width=\linewidth]{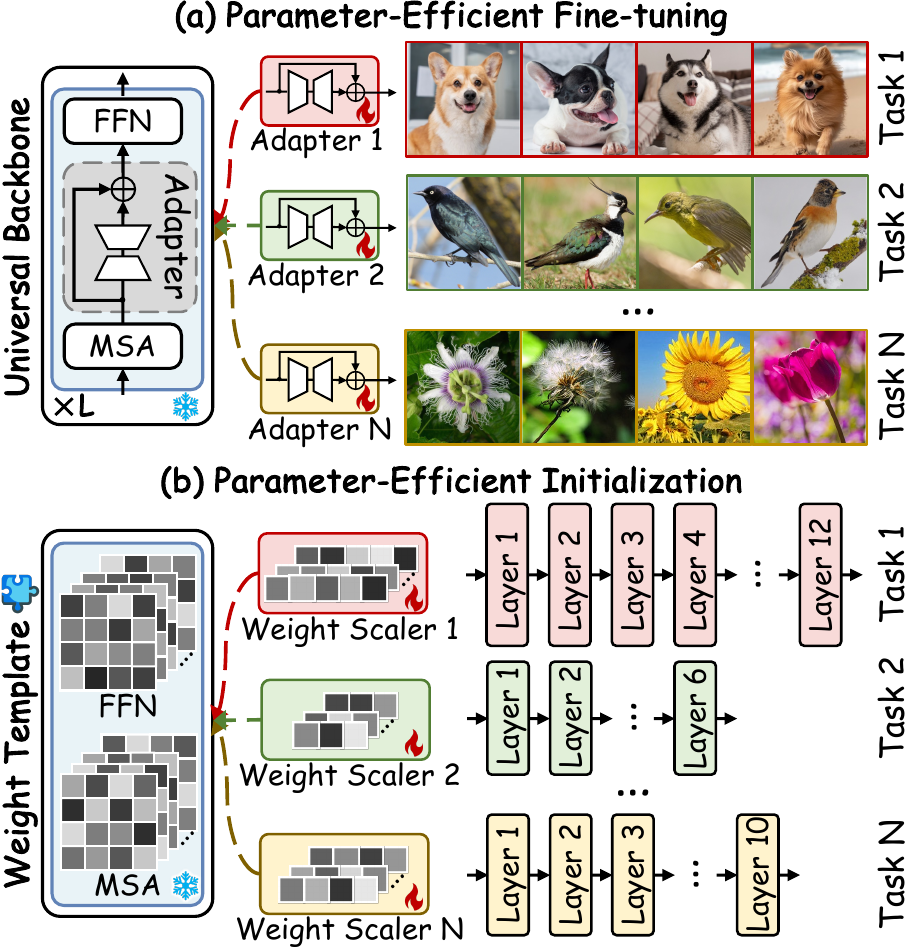}
  \vspace{-0.25in}
  \caption{(a)~Parameter-efficient fine-tuning for multi-task adaptation typically relies on a universal backbone encoding \textbf{\textit{task-agnostic}} knowledge, together with a small number of trainable adapters for \textbf{\textit{task-specific}} adaptation.
  (b)~WeiT reformulates variable-sized model initialization as a multi-task adaptation problem by treating each model size as a distinct task. It enables parameter-efficient initialization by leveraging shared weight templates that encapsulate \textbf{\textit{size-agnostic}} knowledge, together with a small number of trainable weight scalers for \textbf{\textit{size-specific}} adaptation across different model scales.}
  \label{fig:motivation}
  \vspace{-0.15in}
\end{figure}

Recent methods~\cite{wang2023learning, xu2023initializing, shen2026unified} explore scalable model initialization by leveraging pre-trained models for initializing target models with mismatched sizes.
These approaches typically adopt pruning-inspired strategies~\cite{cheng2024survey}, such as layer-wise pruning~\cite{fang2024isomorphic, kim2024bk} or parameter subsampling~\cite{xu2023initializing}, to accommodate smaller models, but may \textit{disrupt the structured knowledge} embedded in pre-trained models. 
Alternatively, distillation-based methods~\cite{wang2021knowledge, gou2021knowledge} enable more flexible knowledge transfer across model sizes, but often require \textit{repeated distillation} for each target configuration, incurring substantial training overhead and limiting their overall efficiency.

This limitation fundamentally stems from the fact that off-the-shelf pre-trained models are not inherently size-aware, as they are \textit{primarily designed for downstream task adaptation}, and do not account for variations in model scale, resulting in learned knowledge being tightly entangled with the specific pre-training configuration~\cite{wang2017growing, elazar2021measuring, wang2021cs}.
Consequently, while such models enable parameter-efficient fine-tuning (e.g., LoRA~\cite{hu2021lora}) for downstream tasks by leveraging a universal backbone~\cite{triantafillou2021learning} that encodes task-agnostic knowledge, as illustrated in Fig.~\ref{fig:motivation}a, they cannot provide scalable initialization for target models.

Thus, a natural question arises: \textit{whether size-agnostic knowledge can be disentangled and encapsulated during pre-training, thereby enabling \textbf{parameter-efficient initialization} for downstream models across scales.}
To this end, we rethink the pre-training process by explicitly incorporating the objective of scalable model initialization, and reformulate variable-sized model initialization as a multi-task adaptation problem (Fig.~\ref{fig:motivation}b) by treating each model size as an independent task.

Following this formulation, we introduce the \textbf{\textit{constraint-based pre-training paradigm}}, a general framework tailored for scalable model initialization.
Unlike conventional pre-training~\cite{ding2023parameter, han2022survey, awais2025foundation}, this paradigm imposes structured constraints~\cite{golub2013matrix, tucker1966some, brewer2003kronecker} on model parameters during optimization, effectively restricting the solution space to suppress size-specific variations and thereby isolating size-agnostic knowledge into compact \textbf{\textit{weight templates}}. 
These templates can then be efficiently reused under the same constraints to initialize downstream models of varying sizes.

Building upon this paradigm, we propose WeiT, which employs Kronecker-based constraints~\cite{brewer2003kronecker} to structure model parameters during pre-training.
Specifically, these constraints represent weight matrices as compositions of \textit{weight templates} via concatenation and weighted aggregation, with adaptive connections rules governed by lightweight \textit{weight scalers}.
Our prior work, WAVE~\cite{feng2024wave} (denoted as WeiT\waveicon), demonstrates the feasibility of learning separate weight templates for multi-head self-attention (MSA) and feed-forward network (FFN) layers.

Here, WeiT extends WeiT\waveicon by constructing \textit{unified weight templates} to enable parameter sharing across heterogeneous components.
Specifically, parameters from all layers and modules are reorganized and concatenated into a single weight matrix $\mathcal{W}$, which is then reconstructed via Kronecker-based constraints (Eq.~\eqref{equ:kro}), with $\mathcal{T}$ serving as the weight templates and $\mathcal{S}$ as lightweight scalers, yielding a compact yet flexible parameter representation.
Moreover, we introduce a \textit{Template Scaling Mechanism}, which applies dimension-wise dropout~\cite{cai2020once} to weight templates during pre-training, thereby enhancing their robustness and adaptability to varying model widths.

Constraint-based pre-training updates model parameters indirectly, with gradients applied to the weight templates and lightweight weight scalers, which are then used to reconstruct the full weight matrix via the Kronecker product.
This enables a \textit{\textbf{once-for-all}} pre-training paradigm, where downstream models are initialized by composing these shared templates in a scaler-driven manner, requiring only a small set of parameters (i.e., weight scalers) to adapt to target model sizes.
Consequently, parameter-efficient initialization is achieved by fixing the weight templates and optimizing only the lightweight scalers, enabling efficient scaling with limited data and negligible computational overhead (e.g., a few hundred gradient steps within a minute-level of wall-clock time).

Our main contributions are as follows:
\begin{itemize}
\item We propose the Constraint-based Pre-training paradigm, a novel  paradigm that explicitly incorporates the objective of scalable model initialization into the pre-training process, approaching variable-sized model initialization from a multi-task adaptation perspective.
\item We propose WeiT, a novel method that leverages Kronecker-based constraints to pre-train shared weight templates, together with lightweight weight scalers for parameter-efficient initialization of target models. 
Building upon WeiT\waveicon, WeiT further enhances template generality and width adaptability, enabling more flexible and efficient initialization across diverse model scales.
\item We introduce a comprehensive benchmark for scalable model initialization across \textsc{Image Classification}, \textsc{Image Generation}, and \textsc{Embodied Control}.
Extensive experiments show that WeiT achieves state-of-the-art performance across diverse tasks and model scales, including both Transformer- and Convolution-based architectures.
\end{itemize}

\section{Related Work}
Model initialization is crucial for convergence speed and final performance of neural networks~\cite{narkhede2022review, arpit2019initialize, huang2020improving}.
Traditional methods typically rely on handcrafted rules to initialize random parameters~\cite{glorot2010understanding, chen2021empirical}, while the emergence of pre-trained models has made fine-tuning the dominant paradigm~\cite{zoph2020rethinking}.
However, the fixed sizes of pre-trained models limit their flexibility, motivating methods that leverage them for scalable initialization of target models, which can be broadly divided into heuristic-driven and knowledge-driven approaches.

\subsubsection{Heuristic-driven Scalable Model Initialization}
\label{sec:heur}
Heuristic-driven approaches transfer parameters based on predefined rules or heuristics; for example, Mimetic Initialization~\cite{trockman2023mimetic} leverages parameter patterns identified from pre-trained models to initialize new ones.
Heur-LG~\cite{wang2022learngene} transfers layers with minimal gradient variation during continual model pre-training. 
Weight Selection~\cite{xu2023initializing} initializes smaller models by sampling parameters from pre-trained models at fixed intervals.
BoT~\cite{shen2026unified} employs Discrete Wavelet Transform (DWT)~\cite{mallat2002theory} and its inverse (IDWT) on weight matrices to enable weight resizing.
Despite their effectiveness, these methods operate directly at the parameter level for model resizing, often introducing structural mismatches that limit parameter flexibility and lead to suboptimal performance or even negative transfer.

\subsubsection{Knowledge-driven Scalable Model Initialization}
\label{sec:know}
Knowledge-driven approaches further refine pre-trained knowledge during cross-scale initialization to preserve, as much as possible, pre-trained model capabilities at the target scale.
Typical strategies include knowledge distillation~\cite{wang2021knowledge, gou2021knowledge}, which aligns outputs or intermediate representations between pre-trained and target models, and pruning-based methods~\cite{cheng2024survey}, which reduce pre-trained model size through structured parameter removal.
Although more effective, these methods incur higher and \textit{repeated computational overhead} for each target model size~\cite{fang2024isomorphic, kim2024bk, xiong2024distilling}, as conventional pre-training tightly couples knowledge to the original scale, making each adaptation inefficient and costly. 
Some approaches introduce simple structural priors, such as layer-wise sharing~\cite{lan2019albert, xia2024transformer, xia2024exploring}, during pre-training to facilitate scalable initialization across depth, but remain limited in adapting to width variations.

Thus, we propose the constraint-based pre-training paradigm, which imposes structured constraints during pre-training to filter and isolate size-agnostic knowledge.
Building on this, WeiT leverages Kronecker-based constraints to learn shared weight templates, enabling efficient and flexible initialization of target models with negligible overhead.

\section{Methodology}
\begin{figure*}[tb]
  \centering
  \includegraphics[width=0.96\linewidth]{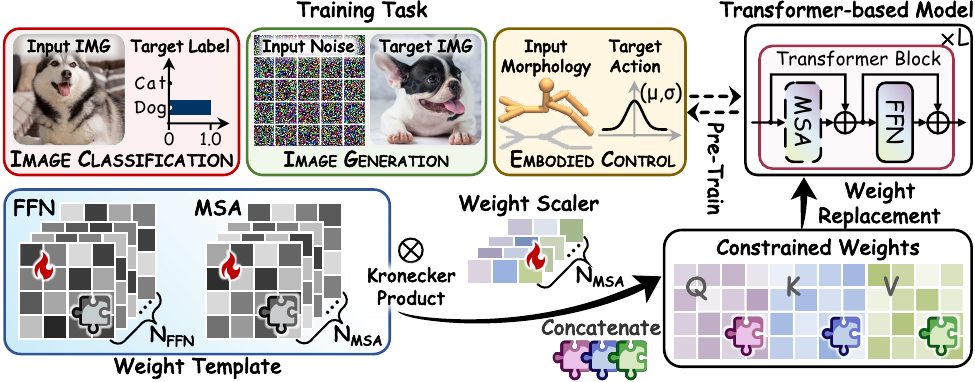}
  \vspace{-0.16in}
  \caption{Overview of the Constraint-based Pre-training Paradigm. Unlike conventional pre-training, it imposes structural constraints (e.g., Kronecker-based constraints) on weight matrices during optimization.
  Specifically, instead of directly updating unconstrained weight matrices in Transformers, we optimize weight templates and reconstruct constrained weight matrices under predefined constraints via lightweight weight scalers, which are then used to replace the corresponding parameters in Transformers (see Algorithm~\ref{alg:algorithm} for details). 
  This indirect optimization process encapsulates size-agnostic knowledge within weight templates, where the imposed structural constraints filter out unstructured variations, thereby enabling effective scaling across model sizes.}
  \label{fig:method-pre}
  \vspace{-0.15in}
\end{figure*}

\subsection{Preliminary}
\subsubsection{Transformer-based Architecture}
A Transformer encoder consists of $L$ stacked layers, each containing a multi-head self-attention (MSA) followed by a feed-forward network (FFN).
In MSA, $H$ attention heads process the input, and their concatenated outputs are projected via the matrix $W_o$:
\begin{equation}
    \text{MSA} = \mathrm{concat}(A_1, \dots, A_H) W_o\,,\quad W_o \in \mathbb{R}^{Hd \times D}.
\label{equ:msa}
\end{equation}
Each attention head $A_i$ computes self-attention from queries $Q_i$, keys $K_i$, and values $V_i \in \mathbb{R}^{N \times d}$ obtained via learnable projections $W_q^i$, $W_k^i$, and $W_v^i \in \mathbb{R}^{D \times d}$:
\begin{equation}
    A_i = \mathrm{softmax}\Big(\frac{Q_i K_i^\top}{\sqrt{d}}\Big)V_i, \quad A_i \in \mathbb{R}^{N \times d},
\end{equation}
where $N$ is the sequence length, $D$ the embedding dimension, and $d$ the head dimension, with $D=Hd$ in standard MSA.
The FFN consists of two linear transformations, $W_{\text{in}} \in \mathbb{R}^{D \times D'}$ and $W_{\text{out}} \in \mathbb{R}^{D' \times D}$, with a GELU~\cite{hendrycks2016gaussian} activation:
\begin{equation}
    \text{FFN}(x) = \mathrm{GELU}(x W_\text{in} + b_1) W_\text{out} + b_2,
\end{equation}
where $b_1$, $b_2$ are bias and $D'$ is the hidden layer dimension, which is typically set to $D'=4D$.

\subsubsection{Problem Formulation}
We consider the problem of initializing target models of varying sizes using knowledge from a single pre-trained model. Let $\mathcal{M}_{\star}$ denote a pre-trained model with parameters $\Theta_{\star}$. Our goal is to efficiently initialize a set of target models $\{\mathcal{M}_t\}_{t=1}^T$ with parameters $\{\Theta_t\}_{t=1}^T$, where each $\mathcal{M}_t$ shares a similar architecture with $\mathcal{M}_{\star}$ but may differ in depth $L$ and width $H$.

Formally, we aim to learn an initialization mapping
\begin{equation}
    f_{\text{init}}: \Theta_{\star} \mapsto \Theta_t, \quad t=1,\dots,T,
\label{equ:f_init}
\end{equation}
such that $\Theta_t$ is derived from $\Theta_{\star}$ as an initialization that inherits transferable knowledge while adapting to the target model scale.
The desired mapping should be both scalable and efficient, enabling variable-sized model initialization with minimal computational overhead.

\subsection{Constraint-based Pre-training Paradigm}
The constraint-based pre-training paradigm formulates scalable model initialization from a multi-task adaptation perspective.
Formally, it realizes $f_{\text{init}}$ in Eq.~\eqref{equ:f_init} by imposing structured constraints $\mathcal{C}$ during the pre-training of $\mathcal{M}_{\star}$, yielding size-agnostic weight templates $\mathcal{T}$ and size-specific weight scalers $\mathcal{S}_\star$, i.e., $(\mathcal{T}, \mathcal{S}_\star) = \mathcal{C}^{-1}(\Theta_{\star})$ (see Fig.~\ref{fig:method-pre}).
The initialization of a target model is then obtained as $\Theta_t = \mathcal{C}(\mathcal{T}, \mathcal{S}_t)$.
Below, we detail the key mechanisms.

\subsubsection{Structural Constraint in Weight Matrix}
\label{sec:wt}
Our prior work, WeiT\waveicon~\cite{feng2024wave}, preliminarily explored imposing constraints on weight matrices via dedicated templates for individual components (i.e., $\mathcal{T}_{\text{MSA}}$ and $\mathcal{T}_{\text{FFN}}$, see Fig.~\ref{fig:method-pre}), which facilitated cross-scale knowledge sharing but overlooked patterns shared across components, limiting template generality.

Thus, WeiT aims to learn unified weight templates shared across all weight matrices (Fig.~\ref{fig:method-init}a). To this end, we first aggregate the primary weight matrices of an $L$-layer Transformer, $\Theta_{\star} = \{W_{q}^{(1\thicksim L)}$, $W_{k}^{(1\thicksim L)}$, $W_{v}^{(1\thicksim L)}$, $W_{o}^{(1\thicksim L)}$, $W_{\text{in}}^{(1\thicksim L)}$, $W_{\text{out}}^{(1\thicksim L)}\}$\footnote{$W_{q}^{(1\thicksim L)}$ denotes $\{W_{q}^{(1)}, \dots, W_{q}^{(L)}\}$ for brevity}, into a unified weight matrix 
\begin{equation}
    \mathcal{W}_\star = \mathrm{concat}(\Theta_\star) \in \mathbb{R}^{L \times P},
\label{eq:agg}
\end{equation}
where each row represents a layer and $P = D \cdot (4Hd + 2D')=12 D \cdot D$.
This aggregation bridges the boundaries between heterogeneous components, enabling unified template learning (see App.~\ref{app:cnn} for generalization to CNNs).

Next, we impose structured constraints on $\mathcal{W}_\star$ for flexible and efficient cross-scale initialization. 
Specifically, WeiT adopts Kronecker-based constraints, which overcome the limitations of prior methods restricted to depth-wise expansion~\cite{lan2019albert, xia2024transformer, xia2024exploring}.
Concretely, the unified weight matrix $\mathcal{W}_\star$ is constrained as
\begin{equation}
    \mathcal{W}_\star \;=\; \mathcal{T} \otimes \mathcal{S}_\star \;=\; \sum_{i=1}^{N} \mathcal{T}_i \otimes \mathcal{S}_{\star, i},
\label{equ:kro}
\end{equation}
where $\otimes$ denotes the Kronecker product. 
Here, $\mathcal{T} = \{\mathcal{T}_i\}_{i=1}^{N}$, where each $\mathcal{T}_i \in \mathbb{R}^{1 \times (r_1\cdot r_2)}$ is a universal \textbf{\textit{Weight Template}} encoding size-agnostic knowledge. Similarly, $\mathcal{S}_\star = \{\mathcal{S}_{\star, i}\}_{i=1}^{N}$, where each $\mathcal{S}_{\star, i} \in \mathbb{R}^{L\times \frac{P}{r_1\cdot r_2}}$ is a lightweight \textbf{\textit{Weight Scaler}} that adaptively composes weight templates via concatenation and weighted aggregation for reconstructing the full weight matrix (see App.~\ref{sec:theoretical_guarantee} for theoretical analysis of expressivity).
Typically, we set $r_1 = r_2 = D$.

In this way, the imposed constraints implicitly encode the objective of effective initialization across diverse model scales into the pre-training process.
To better capture size-agnostic knowledge, we further introduce \textbf{\textit{a low-rank bottleneck}} on weight templates, enforcing $N\cdot r_1 \cdot r_2 \ll L \cdot P$. 
This bottleneck enforces template reuse, promoting maximal sharing of weight templates $\mathcal{T}$ across both depth and width, thereby facilitating extensive cross-scale knowledge sharing while retaining necessary scale-specificity in weight scalers $\mathcal{S}_\star$.

\subsubsection{Template Scaling Mechanism}
WeiT\waveicon provides preliminary support for width expansion by concatenating weight templates along the width dimension. 
However, when initializing narrower models, truncating these templates to match the target width can inadvertently compromise the size-agnostic knowledge encapsulated within them.

Thus, to enhance the flexibility of weight templates for width adaptation, WeiT introduces the \textbf{\textit{Template Scaling Mechanism}}, which applies structured dropout to the weight templates during pre-training (Fig.~\ref{fig:method-init}a), enabling the templates to adapt their effective widths to better accommodate width variations.
Formally, the structured dropout on the weight templates is defined as
\begin{equation}
    \tilde{\mathcal{T}} = M_\mathcal{T} \odot \mathcal{T},
\label{equ:drop}
\end{equation}
where $M_\mathcal{T} \in \{0,1\}^{r_1\times r_2}$ is a structured mask defined by $M_\mathcal{T}(i,j)=1$ for $i \leq r_1'$ and $j \leq r_2'$, and $0$ otherwise, with $(r_1', r_2')$ controlling the effective template width during pre-training. Here, $\odot$ denotes element-wise multiplication.

During pre-training, $(r_1', r_2')$ are randomly sampled at each forward pass, and $\mathcal{W}_\star$ is reconstructed as
\begin{equation}
    \mathcal{W}_\star \;=\; \tilde{\mathcal{T}} \otimes \mathcal{S}_\star \;=\; \sum_{i=1}^{N} \tilde{\mathcal{T}}_i \otimes \mathcal{S}_{\star, i}.
\label{equ:kro_mask}
\end{equation}

By incorporating the Template Scaling Mechanism during pre-training, the weight templates are discouraged from overfitting to a fixed width, encouraging the reorganization of size-agnostic knowledge along the width dimension.

\begin{figure}[tb]
  \centering
  \includegraphics[width=\linewidth]{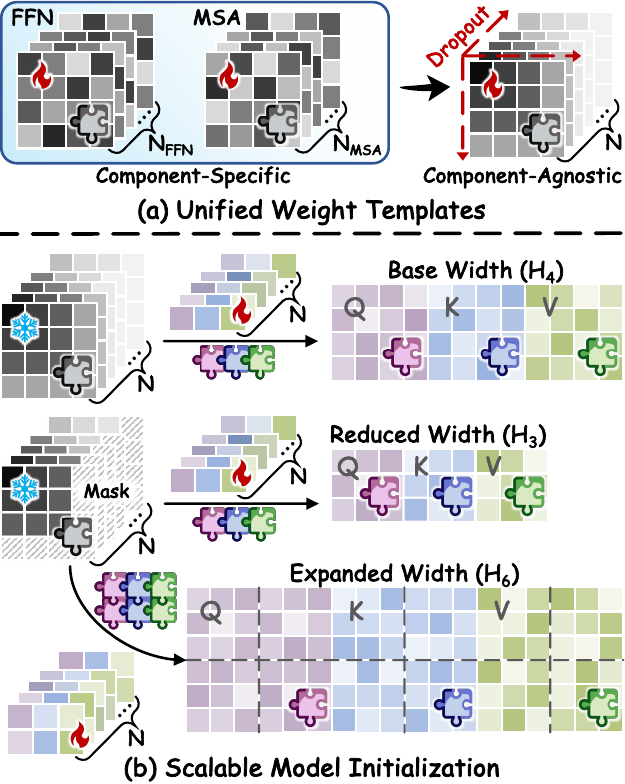}
  \vspace{-0.2in}
  \caption{(a)~WeiT introduces Unified Weight Templates that consolidate component-specific templates from WeiT\waveicon.
  A Template Scaling Mechanism further organizes the knowledge encoded in these templates by selectively activating regions across multiple dimensions via structured dropout.
  (b)~In addition to supporting depth scaling as in WeiT\waveicon, WeiT further enables initialization beyond the base width. 
  Its templates can be adaptively adjusted according to the target width via the Template Scaling Mechanism, thereby enabling initialization of models with both reduced and expanded widths.}
  \label{fig:method-init}
  \vspace{-0.13in}
\end{figure}

\subsubsection{Indirect Parameter Update under Constraints}
\label{sec:know_condense}
During pre-training, model parameters are updated indirectly by optimizing the weight templates $\mathcal{T}$ and scalers $\mathcal{S}_\star$ under the imposed constraints, as detailed in Algorithm~\ref{alg:algorithm} in Appendix.

Specifically, at each iteration, the model parameters $\Theta_\star$ are first reconstructed from the current templates $\mathcal{T}$ and scalers $\mathcal{S}_\star$ according to Eq.~\eqref{equ:kro}. 
Forward and backward propagation are then performed to compute gradients and update $\mathcal{T}$ and $\mathcal{S}_\star$.
Formally, the optimization objective is:
\begin{equation}
    \underset{\mathcal{T}, \mathcal{S}_\star}{\arg\min}\;
    \mathcal{L}(\Theta_\star),
    \quad
    \text{s.t.} \;\Theta_\star = \mathrm{concat}^{-1}(\mathcal{T} \otimes \mathcal{S}_\star),
\label{equ:ob}
\end{equation}
where $\mathrm{concat}^{-1}(\cdot)$ denotes the inverse of Eq.~\eqref{eq:agg}, i.e., the operation that maps the Kronecker-composed matrix back to the original set of model parameters.

The full model parameters are subsequently reassembled from the updated templates and scalers, enabling \textbf{\textit{direct}} updates of the templates and \textbf{\textit{indirect}} updates of the pre-trained model parameters, thereby promoting the encapsulation of size-agnostic knowledge within the weight templates.

For $\mathcal{L}$, we adopt standard task-specific training objectives.
\begin{itemize}
\item \textsc{Image Classification} employs the cross-entropy loss:
\begin{equation}
\mathcal{L}_{\textsc{Cls}} = 
\mathbb{E}_{(x,y)}\left[ - \sum_{c} y_c \log p_{\Theta_\star}(y_c \mid x) \right],
\label{equ:loss_cls}
\end{equation}
where $p_{\Theta_\star}$ denotes the classification network trained to predict class probabilities given input $x$, and $y$ is the corresponding one-hot ground truth label.

\item \textsc{Image Generation} uses a denoising objective:
\begin{equation}
\mathcal{L}_{\textsc{Gen}}=\mathbb{E}_{(z,c,\varepsilon,t)}\left[\|\varepsilon-\varepsilon_{\Theta_\star}(z_t, c, t)\|^{2}_{2}\right],
\label{equ:loss_gen}
\end{equation}
where $\varepsilon_{\Theta_\star}$ denotes the noise prediction network trained to estimate the noise $\varepsilon$ added to the latent variable $z_t$ at timestep $t$ under condition $c$.

\item \textsc{Embodied Control} uses an offline policy distillation objective:
\begin{equation}
\mathcal{L}_{\textsc{Ctl}} =
\mathbb{E}_{(s,a^{\text{ref}})}
\Big[ \sum_{j} (1 - \text{mask}_{j}) \, \log \pi_{\Theta_\star}(a^{\text{ref}}_{j} \mid s) \Big],
\label{equ:loss_ctl}
\end{equation}
where $\pi_{\Theta_\star}$ is the pre-trained policy trained to predict reference actions $a^{\text{ref}}$ given state $s$, and $\text{mask}_{j}$ indicates the action dimensions excluded during distillation.
\end{itemize}

\subsection{Scalable Model Initialization with Weight Templates}
After pre-training, the weight templates $\mathcal{T}$ are frozen to preserve size-agnostic knowledge, while the weight scalers $\mathcal{S}_t$ are instantiated for the target model and efficiently optimized to enable adaptive model initialization.

Specifically, to initialize a target model $\mathcal{M}_t$ with $L_t$ layers and $H_t$ attention heads, let its parameters be $\Theta_t$ and the corresponding aggregated weight matrix be $\mathcal{W}_t \in \mathbb{R}^{L_t \times P_t}$, where $P_t = 12 D_t \cdot D_t$ and $D_t = H_t d$, following the aggregation in Eq.~\eqref{eq:agg}. 
The width of weight templates is then adjusted, if necessary, to $r'_1 \times r'_2$ according to the target model, either by truncation for models with reduced width or by repeated concatenation for models with expanded width (see Fig.~\ref{fig:method-init}b).

The weight scalers are instantiated to match the target model, either via random initialization or by inheriting the pre-trained scalers, yielding $\mathcal{S}_t = \{\mathcal{S}_{t, i}\}_{i=1}^{N}$, where each $\mathcal{S}_{t, i} \in \mathbb{R}^{L_t\times \frac{P_t}{r'_1\cdot r'_2} }$ (Eq.~\eqref{equ:kro}). This defines how target model parameters are reconstructed from the weight templates via Kronecker-based weighted concatenation.

A small subset of data is then used to optimize the weight scalers $\mathcal{S}_t$ via: 
\begin{equation}
    \underset{\mathcal{S}_t}{\arg\min}\;
    \mathcal{L}(\Theta_t),
    \quad
    \text{s.t.} \;\Theta_t = \mathrm{concat}^{-1}(\mathcal{T} \otimes \mathcal{S}_t),
\label{equ:train_S}
\end{equation}
Given the limited parameter count of $\mathcal{S}_t$ (typically a few thousand), convergence is generally reached within a few hundred iterations ($\approx$0.16 epochs, corresponding to only a minute-level of wall-clock time), enabling parameter-efficient model initialization with negligible computational overhead (see App.~\ref{sec:generalization_bounds} for a theoretical generalization analysis).

The model initialization process is completed once the training of $\mathcal{S}_t$ is finalized, after which subsequent training proceeds as usual without imposing additional constraints.

\renewcommand{\arraystretch}{0.88}
\begin{table*}
    \centering
    \setlength{\tabcolsep}{1.5 mm} 
    \caption{\textbf{(Image Classification)} Performance of Variable-sized Model Initialization. Models of different sizes are constructed by varying depth and width, and evaluated using Top-1 Accuracy on ImageNet-1K. ``\textsc{Param.} (M)'' and ``\textsc{FLOPs} (G)'' denote the number of parameters and the computational complexity for each model size. All models are trained for 10 epochs after initialization.}
    \vspace{-0.09in}
    \resizebox{\textwidth}{!}{
        \begin{tabular}{@{}lcccc|cccc|cccc@{}}
        \toprule[1pt]
        & \multicolumn{4}{c}{$H_3$} 
        & \multicolumn{4}{c}{$H_6$} 
        & \multicolumn{4}{c}{$H_{12}$} \\
        \cmidrule(lr){2-5}
        \cmidrule(lr){6-9}
        \cmidrule(l){10-13}
        \textsc{\textbf{Depth Var.}} 
        & $L_4$ & $L_6$ & $L_8$ & $L_{10}$
        & $L_4$ & $L_6$ & $L_8$ & $L_{10}$
        & $L_4$ & $L_6$ & $L_8$ & $L_{10}$ \\
        \cmidrule(r){1-5}
        \cmidrule(lr){6-9}
        \cmidrule(l){10-13}
        \multirow{2}{*}{\textsc{Param.} / \textsc{FLOPs}} & \cellcolor{gray!15}{2.2M} 
        & \cellcolor{gray!15}{3.1M} 
        & \cellcolor{gray!15}{4.0M} 
        & \cellcolor{gray!15}{4.9M} 
        & \cellcolor{gray!15}{7.9M} 
        & \cellcolor{gray!15}{11.5M} 
        & \cellcolor{gray!15}{15.0M} 
        & \cellcolor{gray!15}{18.6M} 
        & \cellcolor{gray!15}{29.9M}
        & \cellcolor{gray!15}{44.1M} 
        & \cellcolor{gray!15}{58.3M} 
        & \cellcolor{gray!15}{72.5M} \\
        
        & \cellcolor{gray!15}{0.8G} 
        & \cellcolor{gray!15}{1.2G} 
        & \cellcolor{gray!15}{1.6G} 
        & \cellcolor{gray!15}{2.0G} 
        & \cellcolor{gray!15}{3.1G} 
        & \cellcolor{gray!15}{4.6G} 
        & \cellcolor{gray!15}{6.1G} 
        & \cellcolor{gray!15}{7.5G} 
        & \cellcolor{gray!15}{11.8G} 
        & \cellcolor{gray!15}{17.5G} 
        & \cellcolor{gray!15}{23.3G} 
        & \cellcolor{gray!15}{29.0G} \\
        \toprule[1pt]
        He Init.~\cite{chen2021empirical} 
                & 34.73 & 40.60 & 43.67 & 46.84 
                & 42.20 & 49.35 & 52.14 & 53.68 
                & 47.89 & 53.13 & 54.44 & 54.95  \\
        Mimetic Init.~\cite{trockman2023mimetic}
                & 35.07 & 40.18 & 43.18 & 46.29 
                & 43.29 & 49.06 & 53.00 & 54.13 
                & 50.16 & 54.34 & 56.50 & 58.49  \\
        Heur-LG~\cite{wang2022learngene}
                & 41.47 & 47.37 & 50.51 & 53.55 
                & 52.33 & 57.32 & 61.67 & 64.35 
                & 60.46 & 68.69 & 72.20 & 73.60 \\
        Auto-LG~\cite{wang2023learngene}
                & 52.38 & 61.80 & 64.56 & 65.88 
                & 63.19 & 70.50 & 72.19 & 73.29 
                & 60.90 & 70.00 & 72.36 & 73.50  \\
        
        Share Init.~\cite{lan2019albert}
                   & 55.16 & 59.83 & 62.52 & 64.25 
                   & 64.95 & 69.66 & 71.65 & 72.65 
                   & 71.72 & 75.30 & 76.44 & 77.40 \\
        TLEG~\cite{xia2024transformer}
             & 55.00 & 60.50 & 62.88 & 64.40 
             & 65.43 & 70.52 & 72.14 & 73.15 
             & 71.56 & 74.85 & 76.24 & 76.99 \\
        \midrule
        \cellcolor{blue!12}{WeiT\waveicon} 
             & \cellcolor{blue!12}{\textbf{58.64}} & \cellcolor{blue!12}{\textbf{63.16}} & \cellcolor{blue!12}{\textbf{65.38}} & \cellcolor{blue!12}{\textbf{66.59}} 
             & \cellcolor{blue!12}{\textbf{68.87}} & \cellcolor{blue!12}{\textbf{72.69}} & \cellcolor{blue!12}{\textbf{74.06}} & \cellcolor{blue!12}{\textbf{74.88}}
             & \cellcolor{blue!12}{\textbf{74.54}} & \cellcolor{blue!12}{\textbf{77.52}} & \cellcolor{blue!12}{\textbf{78.21}} & \cellcolor{blue!12}{\textbf{78.88}} \\
             & \textcolor{mygreen}{$\uparrow$3.48}
                   & \textcolor{mygreen}{$\uparrow$1.36} 
                   & \textcolor{mygreen}{$\uparrow$0.82}
                   & \textcolor{mygreen}{$\uparrow$0.71} 
                   & \textcolor{mygreen}{$\uparrow$3.44} 
                   & \textcolor{mygreen}{$\uparrow$2.17}
                   & \textcolor{mygreen}{$\uparrow$1.87} 
                   & \textcolor{mygreen}{$\uparrow$1.59} 
                   & \textcolor{mygreen}{$\uparrow$2.82}
                   & \textcolor{mygreen}{$\uparrow$2.22} 
                   & \textcolor{mygreen}{$\uparrow$1.77}
                   & \textcolor{mygreen}{$\uparrow$1.48} \\
        \midrule[1pt]
        \toprule[1pt]
        & \multicolumn{4}{c}{$L_3$} 
        & \multicolumn{4}{c}{$L_6$} 
        & \multicolumn{4}{c}{$L_{12}$} \\
        \cmidrule(lr){2-5}
        \cmidrule(lr){6-9}
        \cmidrule(l){10-13}
        \textsc{\textbf{Width Var.}}  & $H_4$ & $H_6$ & $H_8$ & $H_{10}$ 
        & $H_4$ & $H_6$ & $H_8$ & $H_{10}$ 
        & $H_4$ & $H_6$ & $H_8$ & $H_{10}$   \\
        \cmidrule(r){1-5}
        \cmidrule(lr){6-9}
        \cmidrule(l){10-13}
        \multirow{2}{*}{\textsc{Param.} / \textsc{FLOPs}} & \cellcolor{gray!15}{2.9M} 
        & \cellcolor{gray!15}{6.1M} 
        & \cellcolor{gray!15}{10.5M} 
        & \cellcolor{gray!15}{16.1M} 
        & \cellcolor{gray!15}{5.3M} 
        & \cellcolor{gray!15}{11.5M} 
        & \cellcolor{gray!15}{20.0M} 
        & \cellcolor{gray!15}{30.9M} 
        & \cellcolor{gray!15}{10.1M} 
        & \cellcolor{gray!15}{22.2M} 
        & \cellcolor{gray!15}{39.0M} 
        & \cellcolor{gray!15}{60.5M} \\
        & \cellcolor{gray!15}{1.1G} 
        & \cellcolor{gray!15}{2.3G} 
        & \cellcolor{gray!15}{4.1G} 
        & \cellcolor{gray!15}{6.2G} 
        & \cellcolor{gray!15}{2.1G} 
        & \cellcolor{gray!15}{4.6G} 
        & \cellcolor{gray!15}{8.0G} 
        & \cellcolor{gray!15}{12.3G} 
        & \cellcolor{gray!15}{4.2G} 
        & \cellcolor{gray!15}{9.0G} 
        & \cellcolor{gray!15}{15.8G} 
        & \cellcolor{gray!15}{24.4G} \\
        \toprule[1pt]
        He Init.~\cite{chen2021empirical}
            & 34.13 & 38.64 & 40.69 & 44.35
            & 43.11 & 47.38 & 49.11 & 50.96
            & 51.75 & 53.94 & 54.99 & 55.41 \\
        WT-Select~\cite{xu2023initializing}
            & 37.86 & 43.68 & 49.42 & 53.38
            & 49.26 & 56.26 & 63.18 & 66.95
            & 57.38 & 64.56 & 71.80 & 74.43 \\
        Iso. Pruning~\cite{fang2024isomorphic}
            & 38.59 & 47.51 & 50.80 & 54.97
            & 51.05 & 59.90 & 62.26 & 67.57
            & 56.35 & 67.97 & 72.46 & 75.32 \\
        \cellcolor{blue!12}{WeiT}
               & \cellcolor{blue!12}{\textbf{51.46}} & \cellcolor{blue!12}{\textbf{58.64}} & \cellcolor{blue!12}{\textbf{62.11}} & \cellcolor{blue!12}{\textbf{64.20}} 
               & \cellcolor{blue!12}{\textbf{64.20}} & \cellcolor{blue!12}{\textbf{71.21}} & \cellcolor{blue!12}{\textbf{74.25}} & \cellcolor{blue!12}{\textbf{76.41}} 
               & \cellcolor{blue!12}{\textbf{70.95}} & \cellcolor{blue!12}{\textbf{76.98}} & \cellcolor{blue!12}{\textbf{79.34}} & \cellcolor{blue!12}{\textbf{80.75}} \\
                   & \textcolor{mygreen}{$\uparrow$12.87}
                   & \textcolor{mygreen}{$\uparrow$11.13} 
                   & \textcolor{mygreen}{$\uparrow$11.31}
                   & \textcolor{mygreen}{$\uparrow$9.23} 
                   & \textcolor{mygreen}{$\uparrow$13.15} 
                   & \textcolor{mygreen}{$\uparrow$11.31}
                   & \textcolor{mygreen}{$\uparrow$11.07} 
                   & \textcolor{mygreen}{$\uparrow$8.84} 
                   & \textcolor{mygreen}{$\uparrow$13.57}
                   & \textcolor{mygreen}{$\uparrow$9.01} 
                   & \textcolor{mygreen}{$\uparrow$6.88}
                   & \textcolor{mygreen}{$\uparrow$5.43} \\
        \bottomrule[1pt]
        \end{tabular}   
        }
    \label{tab:init_class}
    \vspace{-0.1in}
\end{table*}

\begin{table*}
    \centering 
    \setlength{\tabcolsep}{1.7 mm}
    \caption{\textbf{(Image Classification)} Performance of Initialized Models on Downstream Datasets, evaluated in terms of Top-1 accuracy.}
    \vspace{-0.09in}
    \resizebox{\textwidth}{!}{
        \begin{tabular}{@{}lccccccc |ccccccc@{}}
        \toprule[1pt]
             & \multicolumn{7}{c}{$L_6H_3$} & \multicolumn{7}{c}{$L_6H_6$}\\
             \cmidrule(r){2-8} 
             \cmidrule(l){9-15} 
             & Flower & CUB & Cars & C\small{10} & C\small{100} & Food & iNat.
             & Flower & CUB & Cars & C\small{10} & C\small{100} & Food & iNat. \\
             \midrule[1pt]
             He Init.~\cite{chen2021empirical} 
             & 53.89 & 26.10 & 19.94 & 92.40 & 68.33 & 68.41 & 52.28
             & 57.23 & 27.27 & 23.77 & 93.96 & 66.50 & 70.63 & 54.03 
             \\
             Mimetic Init.~\cite{trockman2023mimetic} 
             & 52.12 & 34.95 & 20.53 & 88.90 & 63.39 & 66.90 & 49.04 
             & 57.42 & 39.63 & 34.24 & 91.59 & 65.68 & 67.08 & 52.18 \\ 
             Heur-LG~\cite{wang2022learngene} 
             & 64.71 & 44.60 & 37.72 & 93.99 & 71.12 & 74.71 & 57.36 
             & 69.08 & 47.96 & 51.18 & 95.07 & 72.79 & 76.84 & 59.31 \\
             Auto-LG~\cite{wang2023learngene} 
             & 93.53 & 71.40 & 83.50 & 96.39 & 77.06 & 81.73 & 62.51
             & 96.39 & 75.06 & 88.19 & 97.32 & 80.99 & 84.64 & 67.00 \\
            Share Init.~\cite{lan2019albert} 
            & 92.39 & 70.09 & 82.08 & 96.00 & 77.23 & 81.23 & 63.20 
            & 94.10 & 72.44 & 87.15 & 96.49 & 78.52 & 82.95 & 62.97 \\
             TLEG~\cite{xia2024transformer} 
             & 91.04 & 69.50 & 78.16 & 96.06 & 76.93 & 81.92 & 63.40 
             & 93.74 & 72.63 & 87.22 & 97.18 & 80.24 & 84.87 & 66.50  \\
             \cellcolor{blue!12}{WeiT\waveicon} & \cellcolor{blue!12}{\textbf{94.89}} & \cellcolor{blue!12}{\textbf{74.77}} & \cellcolor{blue!12}{\textbf{84.44}} & \cellcolor{blue!12}{\textbf{96.57}} & \cellcolor{blue!12}{\textbf{80.70}} & \cellcolor{blue!12}{\textbf{83.83}} & \cellcolor{blue!12}{\textbf{65.25}} 
                  & \cellcolor{blue!12}{\textbf{96.89}} & \cellcolor{blue!12}{\textbf{78.10}} & \cellcolor{blue!12}{\textbf{89.43}} & \cellcolor{blue!12}{\textbf{97.38}} & \cellcolor{blue!12}{\textbf{83.18}} & \cellcolor{blue!12}{\textbf{85.53}} & \cellcolor{blue!12}{\textbf{67.62}} \\
             & \textcolor{mygreen}{$\uparrow$1.36} & \textcolor{mygreen}{$\uparrow$3.37} & \textcolor{mygreen}{$\uparrow$0.94} & \textcolor{mygreen}{$\uparrow$0.18} & \textcolor{mygreen}{$\uparrow$3.47} & \textcolor{mygreen}{$\uparrow$1.91} & \textcolor{mygreen}{$\uparrow$1.85} 
                  & \textcolor{mygreen}{$\uparrow$0.50} & \textcolor{mygreen}{$\uparrow$3.04} & \textcolor{mygreen}{$\uparrow$1.24} & \textcolor{mygreen}{$\uparrow$0.06} & \textcolor{mygreen}{$\uparrow$2.19} & \textcolor{mygreen}{$\uparrow$0.66} & \textcolor{mygreen}{$\uparrow$0.62} \\
             \midrule
             Full FT 
             & 95.35 & 75.15 & 86.48 & 96.62 & 80.22 & 83.95 & 66.86 
             & 96.39 & 77.03 & 89.43 & 97.52 & 82.82 & 85.58 & 69.31 \\
             \bottomrule[1pt]
        \end{tabular}
    \label{tab:downstream_class}
    }
    \vspace{-0.07in}
\end{table*}

\section{Experiments}
In this section, we establish comprehensive benchmarks on \textsc{Image Classification} (Sec.~\ref{sec:result_classi}), \textsc{Image Generation} (Sec.~\ref{sec:result_gen}), and \textsc{Embodied Control} (Sec.~\ref{sec:result_cont}) to systematically evaluate the initialization capability of weight templates across varying model sizes and diverse downstream tasks.
We further assess the architectural scalability of WeiT by extending it to Convolution-based models (e.g., ConvNeXt~\cite{woo2023convnext}) (Sec.~\ref{sec:cnn}), and analyze the optimization dynamics of WeiT-initialized models under full training (Sec.~\ref{sec:full_train}).
Finally, we conduct a systematic analysis of weight templates (Sec.~\ref{sec:ablation}), examining key design choices and their impact on performance, followed by visualizations of the knowledge encapsulated within these templates (Sec.~\ref{sec:visual}).

\subsection{Main Results on Image Classification}
\label{sec:result_classi}

\subsubsection{Experimental Setup}
\begin{itemize}
\item \textbf{Basic Settings.} 
We adopt DeiT~\cite{touvron2021training} as the base architecture and use DeiT-B (i.e., $L_{12}H_{12}$) for constraint-based pre-training of WeiT on ImageNet-1K~\cite{deng2009imagenet}. 
Pre-training is conducted for 300 epochs with a batch size of 1024 and a learning rate of $5\times10^{-4}$, using the AdamW optimizer with a cosine learning rate scheduler.

\item \textbf{Evaluation.} 
For multi-scale initialization, we vary the model depth from $L_4$ to $L_{24}$ and width from $H_4$ to $H_{24}$, covering a broad range of DeiT configurations across both smaller and larger model scales.
For knowledge transfer, we further evaluate on diverse downstream datasets, including Oxford Flowers~\cite{nilsback2008automated}, CUB-200-2011~\cite{wah2011caltech}, Stanford Cars~\cite{gebru2017fine}, CIFAR-10/100~\cite{krizhevsky09}, Food-101~\cite{bossard2014food}, and iNaturalist-2019~\cite{tan2019herbarium} (see App.~\ref{app:dataset_cls} for details).
\end{itemize}

\subsubsection{Initialization across Model Scales}
Table~\ref{tab:init_class} presents the results of initializing models of varying sizes for \textsc{Image Classification}, highlighting that knowledge-driven scalable initialization methods generally outperform heuristic-driven approaches. Under depth scaling, Share Init.~\cite{lan2019albert} and TLEG~\cite{xia2024transformer} improve over Mimetic Init~\cite{trockman2023mimetic}. by reusing specific layers; however, their rigid layer-wise sharing may limit adaptability across diverse model depths.
In contrast, WeiT\waveicon constructs unified weight templates and employs a small set of trainable weight scalers to flexibly adapt the reconstruction of weight templates for each target depth, thereby preserving high representational fidelity.

Furthermore, WeiT’s Template Scaling Mechanism enhances the adaptability of weight templates, enabling flexible width scaling of models. 
When initializing models with varying widths, parameters are efficiently constructed by concatenating weight templates via the Kronecker product.
This design allows WeiT to consistently outperform existing knowledge transfer methods, such as pruning-based Iso. pruning~\cite{fang2024isomorphic}, while incurring lower initialization overhead (see Table~\ref{tab:init_class}).

\subsubsection{Transferability to Downstream Datasets}
The knowledge encapsulated in weight templates is sufficiently general to transfer across diverse downstream datasets (Table~\ref{tab:downstream_class}), with WeiT\waveicon consistently delivering substantial improvements compared to Share Init.~\cite{lan2019albert} and TLEG~\cite{xia2024transformer}. 
By contrast, Mimetic Init.~\cite{trockman2023mimetic} may underperform relative to He Init.~\cite{chen2021empirical} on certain datasets (e.g., Food-101 and iNaturalist-2019), highlighting the limited generality of its heuristic-based initialization.

Moreover, small-scale datasets (e.g., Oxford Flowers and Stanford Cars) offer limited training data for large models, as observed with He Init.~\cite{chen2021empirical} and Mimetic Init.~\cite{trockman2023mimetic}, highlighting the critical role of effective knowledge transfer.
WeiT leverages the structured knowledge encapsulated in weight templates to enable adaptive initialization, thereby enhancing data efficiency under scarce data conditions.

\begin{table*}
    \centering
    \setlength{\tabcolsep}{0.9 mm} 
        \caption{\textbf{(Image Generation)} Performance of Variable-sized Model Initialization. Models of different sizes are constructed by varying depth and width, and evaluated using $\text{FID}_\text{(IS)}$ on ImageNet-1K. ``\textsc{Param.} (M)'' and ``\textsc{FLOPs} (G)'' denote the number of parameters and the computational complexity for each model size. All models are trained 100K steps after initialization.}
        \vspace{-0.05in}
        \resizebox{\textwidth}{!}{
        \begin{tabular}{@{}lcccc|cccc@{}}
        \toprule[1pt]
        & \multicolumn{4}{c}{$H_{12}$} 
        & \multicolumn{4}{c}{$H_{16}$} \\
        \cmidrule(r){2-5}
        \cmidrule(l){6-9}
        \textsc{\textbf{Depth Var.}} 
        & $L_4$ & $L_6$ & $L_8$ & $L_{10}$ 
        & $L_4$ & $L_6$ & $L_8$ & $L_{10}$  \\

        \cmidrule(r){1-5}
        \cmidrule(l){6-9}
        \multirow{2}{*}{\textsc{Param.} / \textsc{FLOPs}} 
        & \cellcolor{gray!15}{45.3M}
        & \cellcolor{gray!15}{66.6M} 
        & \cellcolor{gray!15}{87.8M}
        & \cellcolor{gray!15}{109.1M}
        & \cellcolor{gray!15}{80.1M}
        & \cellcolor{gray!15}{117.8M}
        & \cellcolor{gray!15}{155.6M}
        & \cellcolor{gray!15}{193.4M} \\
        
        & \cellcolor{gray!15}{14.6G}
        & \cellcolor{gray!15}{21.8G}
        & \cellcolor{gray!15}{29.1G}
        & \cellcolor{gray!15}{36.4G}
        
        & \cellcolor{gray!15}{25.9G}
        & \cellcolor{gray!15}{38.8G}
        & \cellcolor{gray!15}{51.7G}
        & \cellcolor{gray!15}{64.6G} \\
        
        \toprule[1pt]
        He Init.~\cite{chen2021empirical} 
                    & $\text{87.23}_\text{ (16.09)}$
                    & $\text{80.37}_\text{ (17.20)}$ 
                    & $\text{71.39}_\text{ (19.41)}$
                    & $\text{70.73}_\text{ (19.04)}$ 
                    & $\text{78.46}_\text{ (17.98)}$ 
                    & $\text{72.57}_\text{ (19.19)}$ 
                    & $\text{64.91}_\text{ (20.84)}$
                    & $\text{59.64}_\text{ (22.61)}$ \\
                    
        Mimetic Init.~\cite{trockman2023mimetic} 
                    & $\text{81.76}_\text{ (16.69)}$ 
                    & $\text{79.87}_\text{ (18.58)}$ 
                    & $\text{72.04}_\text{ (19.70)}$
                    & $\text{66.98}_\text{ (21.00)}$ 
                    & $\text{77.55}_\text{ (18.15)}$ 
                    & $\text{69.79}_\text{ (19.31)}$ 
                    & $\text{64.45}_\text{ (22.13)}$ 
                    & $\text{64.55}_\text{ (21.55)}$  \\
                    
        Heur-LG~\cite{wang2022learngene} 
                    & $\text{84.14}_\text{ (16.21)}$ 
                    & $\text{70.84}_\text{ (19.94)}$ 
                    & $\text{62.57}_\text{ (23.52)}$
                    & $\text{60.88}_\text{ (23.40)}$
                    & $\text{81.37}_\text{ (17.73)}$
                    & $\text{65.49}_\text{ (22.96)}$
                    & $\text{58.41}_\text{ (25.12)}$ 
                    & $\text{55.34}_\text{ (26.58)}$ \\
                    
        Auto-LG~\cite{wang2023learngene} 
                    & $\text{81.63}_\text{ (18.18)}$ 
                    & $\text{66.70}_\text{ (22.59)}$ 
                    & $\text{64.07}_\text{ (24.09)}$
                    & $\text{59.80}_\text{ (25.32)}$
                    & $\text{77.66}_\text{ (18.73)}$ 
                    & $\text{68.03}_\text{ (22.46)}$ 
                    & $\text{60.42}_\text{ (25.92)}$ 
                    & $\text{59.98}_\text{ (25.98)}$ \\
                
        Share Init.~\cite{lan2019albert} 
                    & $\text{66.87}_\text{ (22.35)}$ 
                    & $\text{59.03}_\text{ (24.61)}$ 
                    & $\text{53.43}_\text{ (26.80)}$
                    & $\text{51.06}_\text{ (28.00)}$ 
                    & $\text{58.54}_\text{ (24.89)}$ 
                    & $\text{46.26}_\text{ (30.88)}$ 
                    & $\text{43.76}_\text{ (32.25)}$ 
                    & $\text{41.78}_\text{ (33.27)}$ \\
        Laptop-Diff~\cite{zhang2024laptop}  
                    & $\text{105.9}_\text{ (12.99)}$ 
                    & $\text{68.71}_\text{ (21.00)}$ 
                    & $\text{52.73}_\text{ (26.97)}$
                    & $\text{52.57}_\text{ (27.46)}$
                    & $\text{107.9}_\text{ (12.92)}$ 
                    & $\text{63.02}_\text{ (23.33)}$ 
                    & $\text{47.84}_\text{ (30.52)}$ 
                    & $\text{47.50}_\text{ (31.09)}$ \\
        TLEG~\cite{xia2024transformer} 
                    & $\text{62.88}_\text{ (22.78)}$ 
                    & $\text{54.97}_\text{ (26.76)}$ 
                    & $\text{49.04}_\text{ (28.76)}$
                    & $\text{47.22}_\text{ (30.23)}$ 
                    & $\text{53.00}_\text{ (27.99)}$ 
                    & $\text{46.69}_\text{ (30.85)}$ 
                    & $\text{44.32}_\text{ (32.23)}$ 
                    & $\text{41.15}_\text{ (34.80)}$  \\    
        FINE~\cite{xie2024fine} 
               & $\text{57.47}_\text{ (24.52)}$
               & $\text{51.58}_\text{ (27.52)}$
               & $\text{45.34}_\text{ (30.46)}$
               & $\text{42.33}_\text{ (32.34)}$ 
               & $\text{48.72}_\text{ (28.27)}$
               & $\text{44.38}_\text{ (31.41)}$
               & $\text{41.24}_\text{ (34.25)}$
               & $\text{36.53}_\text{ (36.79)}$  \\
        \midrule
        \cellcolor{blue!12}{WeiT}~\cite{feng2024wave} 
                    & \cellcolor{blue!12}{$\textbf{55.41}_\text{  (\textbf{26.03})}$} 
                    & \cellcolor{blue!12}{$\textbf{48.40}_\text{  (\textbf{29.72})}$} 
                    & \cellcolor{blue!12}{$\textbf{44.54}_\text{  (\textbf{31.98})}$} 
                    & \cellcolor{blue!12}{$\textbf{41.67}_\text{  (\textbf{33.94})}$} 
                    & \cellcolor{blue!12}{$\textbf{47.90}_\text{  (\textbf{30.97})}$} 
                    & \cellcolor{blue!12}{$\textbf{44.31}_\text{  (\textbf{33.04})}$} 
                    & \cellcolor{blue!12}{$\textbf{37.90}_\text{  (\textbf{36.99})}$}  
                    & \cellcolor{blue!12}{$\textbf{36.29}_\text{  (\textbf{39.33})}$} \\
        & \textcolor{mygreen}{$\downarrow\text{2.06}_{\uparrow\text{1.51}}$} 
        & \textcolor{mygreen}{$\downarrow \text{3.18}_{\uparrow\text{2.20}}$} 
        & \textcolor{mygreen}{$\downarrow \text{0.80}_{\uparrow\text{1.52}}$} 
        & \textcolor{mygreen}{$\downarrow \text{0.66}_{\uparrow\text{1.60}}$} 
        & \textcolor{mygreen}{$\downarrow \text{0.82}_{\uparrow\text{2.70}}$} 
        & \textcolor{mygreen}{$\downarrow \text{0.07}_{\uparrow\text{1.63}}$} 
        & \textcolor{mygreen}{$\downarrow \text{3.34}_{\uparrow\text{2.74}}$} 
        & \textcolor{mygreen}{$\downarrow \text{0.24}_{\uparrow\text{2.54}}$} \\
        \bottomrule[1pt]
        \toprule[1pt]
        & \multicolumn{4}{c}{$L_{6}$} 
        & \multicolumn{4}{c}{$L_{12}$} \\
        \cmidrule(r){2-5}
        \cmidrule(l){6-9}
        \textsc{\textbf{Width Var.}} 
        & $H_9$ & $H_{11}$ & $H_{13}$ & $H_{15}$ 
        & $H_9$ & $H_{11}$ & $H_{13}$ & $H_{15}$   \\

        \cmidrule(r){1-5}
        \cmidrule(l){6-9}
        \multirow{2}{*}{\textsc{Param.} / \textsc{FLOPs}}
        & \cellcolor{gray!15}{37.6M} 
        & \cellcolor{gray!15}{56.0M} 
        & \cellcolor{gray!15}{78.0M} 
        & \cellcolor{gray!15}{103.6M} 
        & \cellcolor{gray!15}{73.5M} 
        & \cellcolor{gray!15}{109.6M} 
        & \cellcolor{gray!15}{152.8M} 
        & \cellcolor{gray!15}{203.3M} \\
        
        & \cellcolor{gray!15}{12.3G} 
        & \cellcolor{gray!15}{18.3G} 
        & \cellcolor{gray!15}{25.6G} 
        & \cellcolor{gray!15}{34.1G} 
        & \cellcolor{gray!15}{24.5G} 
        & \cellcolor{gray!15}{36.7G} 
        & \cellcolor{gray!15}{51.2G} 
        & \cellcolor{gray!15}{68.1G} \\
        
       \toprule[1pt]
       He Init.~\cite{chen2021empirical}
                    & $\text{80.96}_\text{ (16.46)}$
                    & $\text{80.48}_\text{ (16.85)}$
                    & $\text{74.50}_\text{ (18.21)}$
                    & $\text{69.32}_\text{ (19.54)}$
                    & $\text{70.87}_\text{ (18.74)}$
                    & $\text{78.24}_\text{ (17.73)}$
                    & $\text{71.92}_\text{ (19.36)}$
                    & $\text{64.61}_\text{ (21.50)}$ \\
       
       BK-SDM~\cite{kim2024bk}
                    & $\text{83.17}_\text{ (17.18)}$
                    & $\text{72.53}_\text{ (19.70)}$
                    & $\text{66.58}_\text{ (21.60)}$
                    & $\text{67.11}_\text{ (21.43)}$
                    & $\text{64.50}_\text{ (20.93)}$
                    & $\text{61.91}_\text{ (23.37)}$
                    & $\text{58.34}_\text{ (24.88)}$
                    & $\text{53.14}_\text{ (26.70)}$\\
        
        WT-Select~\cite{xu2023initializing} 
                    & $\text{79.64}_\text{ (17.43)}$
                    & $\text{82.11}_\text{ (17.64)}$
                    & $\text{80.13}_\text{ (18.83)}$
                    & $\text{75.13}_\text{ (19.47)}$
                    & $\text{55.00}_\text{ (26.60)}$
                    & $\text{45.48}_\text{ (34.27)}$
                    & $\text{46.22}_\text{ (35.51)}$
                    & $\text{37.14}_\text{ (42.31)}$\\
       \cellcolor{blue!12}{WeiT} 
               & \cellcolor{blue!12}{$\textbf{48.94}_\text{ (\textbf{30.68})}$}
               & \cellcolor{blue!12}{$\textbf{45.40}_\text{ (\textbf{34.83})}$}
               & \cellcolor{blue!12}{$\textbf{38.14}_\text{ (\textbf{41.08})}$}
               & \cellcolor{blue!12}{$\textbf{36.45}_\text{ (\textbf{43.07})}$}
               & \cellcolor{blue!12}{$\textbf{45.35}_\text{ (\textbf{33.25})}$}
               & \cellcolor{blue!12}{$\textbf{35.77}_\text{ (\textbf{40.93})}$}
               & \cellcolor{blue!12}{$\textbf{30.10}_\text{ (\textbf{47.59})}$}
               & \cellcolor{blue!12}{$\textbf{27.11}_\text{ (\textbf{52.78})}$} \\
        & \textcolor{mygreen}{$\downarrow \text{30.70}_{\uparrow\text{13.25}}$} 
        & \textcolor{mygreen}{$\downarrow \text{27.13}_{\uparrow\text{15.13}}$} 
        & \textcolor{mygreen}{$\downarrow \text{28.44}_{\uparrow\text{19.48}}$} 
        & \textcolor{mygreen}{$\downarrow \text{30.66}_{\uparrow\text{21.64}}$} 
        & \textcolor{mygreen}{$\downarrow \text{9.65}_{\uparrow\text{6.65}}$} 
        & \textcolor{mygreen}{$\downarrow \text{9.71}_{\uparrow\text{6.66}}$} 
        & \textcolor{mygreen}{$\downarrow \text{16.12}_{\uparrow\text{12.08}}$} 
        & \textcolor{mygreen}{$\downarrow \text{10.03}_{\uparrow\text{10.47}}$} \\

        \bottomrule[1pt]
        \end{tabular}
        }
    \label{tab:init_gene}
\vspace{-0.05in}
\end{table*}

\begin{table*}[t]
    \centering
    \setlength{\tabcolsep}{1 mm}
    \caption{\textbf{(Image Generation)} Performance of Initialized Models on Downstream Datasets, evaluated using FID for natural image datasets (i.e., CelebA, Bedroom, and Church) and FDD for non-natural image ones (i.e., Hubble, MRI, and Pokemon).}
    \vspace{-0.05in}
    \resizebox{\textwidth}{!}{
        \begin{tabular}{@{}lccc ccc | ccc ccc@{}}
        \toprule[1pt]
             & \multicolumn{6}{c}{$L_6H_{12}$}
             & \multicolumn{6}{c}{$L_6H_{16}$} \\
             \cmidrule(r){2-7}
             \cmidrule(l){8-13}
             & CelebA & Bedroom & Church 
             & Hubble & MRI & Pokemon 
              & CelebA & Bedroom & Church 
             & Hubble & MRI & Pokemon \\
             \midrule[1pt]
             He Init.~\cite{chen2021empirical} 
             & 18.57 & 42.90 & 41.01 
             & 0.320 & 0.170 & 0.897 
             & 14.55 & 32.88 & 24.27
             & 0.235 & 0.119 & 0.925 \\
             Mimetic Init.~\cite{trockman2023mimetic}
             & 16.87 & 30.48 & 33.65 
             & 0.281 & 0.180 & 0.902 
             & 11.66 & 29.66 & 25.00 
             & 0.271 & 0.111 & 0.920 \\
             Heur-LG~\cite{wang2022learngene}
             & 13.23 & 36.98 & 29.13
             & 0.293 & 0.127 & 0.865 
             & 10.84 & 24.42 & 17.09 
             & 0.314 & 0.099 & 0.919 \\
             Auto-LG~\cite{wang2023learngene}
             & 15.02 & 46.56 & 44.15 
             & 0.302 & 0.110 & 0.705 
             & 16.54 & 38.98 & 31.58 
             & 0.270 & 0.148 & 0.764 \\
             Share Init.~\cite{lan2019albert}  
             & 9.11 & 25.47 & 22.49 
             & 0.190 & 0.057 & 0.463 
             & 9.40 & 17.90 & 19.88 
             & 0.119 & 0.047 & 0.421 \\
             Laptop-Diff~\cite{zhang2024laptop}
             & 12.62 & 22.85 & 24.73 
             & 0.153 & 0.063 & 0.466 
             & 10.52 & 27.01 & 25.19 
             & 0.141 & 0.051 & 0.478 \\
             TLEG~\cite{xia2024transformer}
             & 8.27 & 20.43 & 19.30 
             & 0.226 & 0.057 & 0.428 
             & 10.91 & 19.43 & 18.29 
             & 0.124 & 0.052 & 0.412 \\
             FINE~\cite{xie2024fine}
             & 7.99 & 17.83 & 17.29 
             & 0.119 & 0.049 & 0.407 
             & 8.41 & 14.90 & 15.80 
             & 0.101 & 0.041 & 0.380 \\
             \cellcolor{blue!12}{WeiT}~\cite{feng2024wave} 
             & \cellcolor{blue!12}{\textbf{7.36}} & \cellcolor{blue!12}{\textbf{17.62}} & \cellcolor{blue!12}{\textbf{16.52}}
             & \cellcolor{blue!12}{\textbf{0.114}} & \cellcolor{blue!12}{\textbf{0.043}} & \cellcolor{blue!12}{\textbf{0.400}}
             & \cellcolor{blue!12}{\textbf{5.68}} & \cellcolor{blue!12}{\textbf{14.63}} & \cellcolor{blue!12}{\textbf{15.31}} 
             & \cellcolor{blue!12}{\textbf{0.100}} & \cellcolor{blue!12}{\textbf{0.038}} & \cellcolor{blue!12}{\textbf{0.378}} \\
             & \textcolor{mygreen}{$\downarrow$0.63} 
               & \textcolor{mygreen}{$\downarrow$0.21} 
               & \textcolor{mygreen}{$\downarrow$0.77} 
               & \textcolor{mygreen}{$\downarrow$0.005} 
               & \textcolor{mygreen}{$\downarrow$0.006}
               & \textcolor{mygreen}{$\downarrow$0.007}
            & \textcolor{mygreen}{$\downarrow$2.73} 
               & \textcolor{mygreen}{$\downarrow$0.27} 
               & \textcolor{mygreen}{$\downarrow$0.49} 
               & \textcolor{mygreen}{$\downarrow$0.001} 
               & \textcolor{mygreen}{$\downarrow$0.003}
               & \textcolor{mygreen}{$\downarrow$0.002}\\
           \midrule
           Full FT 
           & 9.97 & 24.43 & 20.65 
           & 0.148 & 0.060 & 0.418 
           & 8.65 & 19.58 & 19.10 
           & 0.124 & 0.048 & 0.421 \\
           \bottomrule[1pt]
        \end{tabular}
        }
    \label{tab:downstream_gen}
    \vspace{-0.1in}
\end{table*}

\subsection{Main Results on Image Generation}
\label{sec:result_gen}

\subsubsection{Experimental Setup}
\begin{itemize}
\item \textbf{Basic Settings.}
We focus on class-conditional generation and adopt Diffusion Transformers (DiTs)~\cite{peebles2023scalable}  as the backbone, using DiT-L (i.e., $L_{12}H_{16}$) for constraint-based pre-training of WeiT on ImageNet-1K~\cite{deng2009imagenet}. Models employ a latent patch size of $p=2$ and process images at $256 \times 256$ resolution. 
Pre-training is conducted for 600K steps with a batch size of 64, a fixed learning rate of $1 \times 10^{-4}$, and optimized using AdamW.

\item \textbf{Evaluation.}
For multi-scale initialization, we vary depth from $L_4$ to $L_{24}$ and width from $H_9$ to $H_{32}$.
Performance is measured using Fr\'{e}chet Inception Distance (FID)~\cite{heusel2017gans} and Inception Score (IS)~\cite{salimans2016improved}.
To assess transferability, we conduct experiments on diverse domains—including CelebA-HQ~\cite{huang2018introvae}, LSUN-Bedroom, LSUN-Church~\cite{wang2017knowledge}, Hubble, MRI, and Pokemon—that differ substantially from the pre-training dataset (see App.~\ref{app:dataset_gen} for details).
\end{itemize}

\subsubsection{Initialization across Model Scales}
Table~\ref{tab:init_gene} presents the results of initializing models of varying sizes for \textsc{Image Generation}, with WeiT consistently outperforming existing methods across all scales, achieving notably lower FID scores, particularly for models with varying widths.
Distillation-based methods (e.g., Laptop-Diff~\cite{zhang2024laptop}) and pruning-based approaches (e.g., BK-SDM~\cite{kim2024bk}) support flexible knowledge transfer but incur considerable overhead for each target size, constraining efficiency. 
In contrast, WeiT achieves efficient adaptation with only a few hundred optimization steps.

Moreover, when the target model (e.g., $L_4H_{12}$ or $L_6H_{9}$) deviates substantially from the pre-trained configuration (i.e., $L_{12}H_{16}$), existing methods suffer from disrupted layer-wise denoising hierarchy and temporal coherence, resulting in suboptimal generative alignment—particularly under width scaling.
In contrast, WeiT provides flexible and stable initialization across both depths and widths, enabling reliable initialization even for very small models, a property essential for diffusion models that are highly sensitive to initialization quality.

\subsubsection{Transferability to Downstream Datasets}
As shown in Table~\ref{tab:downstream_gen}, WeiT facilitates flexible model initialization across diverse downstream image generation datasets and consistently outperforms existing methods, demonstrating that weight templates obtained via our constraint-based pre-training encapsulate knowledge that is both size-agnostic and, to a considerable extent, domain-agnostic.

Remarkably, WeiT-initialized models even surpass direct fine-tuning from pre-trained models (i.e., Full FT), further underscoring WeiT’s superior data efficiency. 
This indicates that transferring more parameters does not necessarily improve performance~\cite{feng2024transferring}, particularly under substantial domain gaps (e.g., Hubble and MRI), where redundant knowledge may hinder adaptability.

\begin{figure*}[!t]
    \centering
    \includegraphics[width=\linewidth]{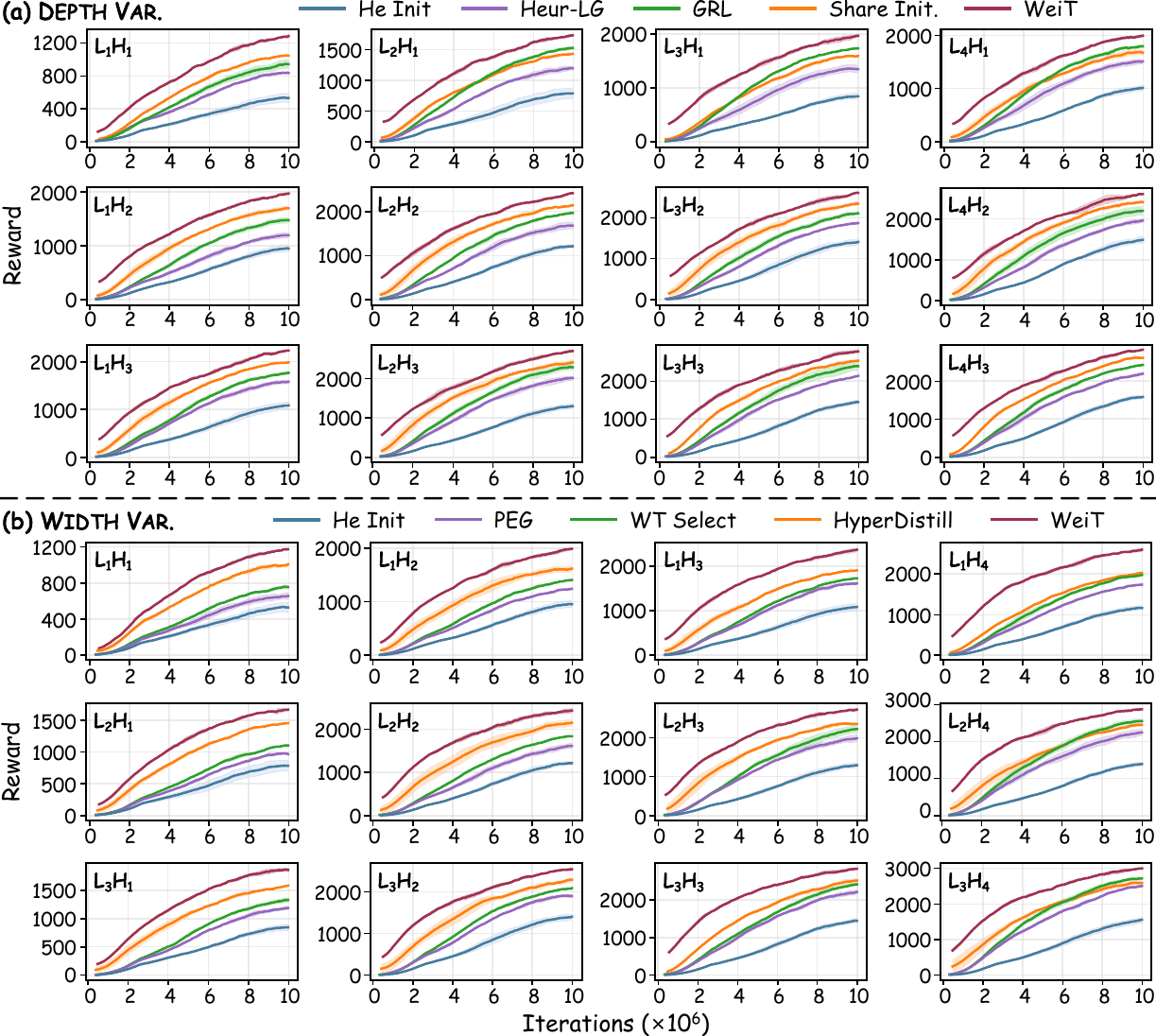}
    \vspace{-0.25in}
    \caption{\textbf{\textsc{(Embodied Control)}} Performance of Variable-sized Model Initialization. 
    Models are scaled by varying depth and width and evaluated on Flat Terrain with novel morphologies using cumulative reward. All models are trained for $1\times 10^7$ iterations after initialization.}
    \label{fig:control_scale}
    \vspace{-0.15in}
\end{figure*}

\begin{figure*}[!t]
    \centering
    \includegraphics[width=\linewidth]{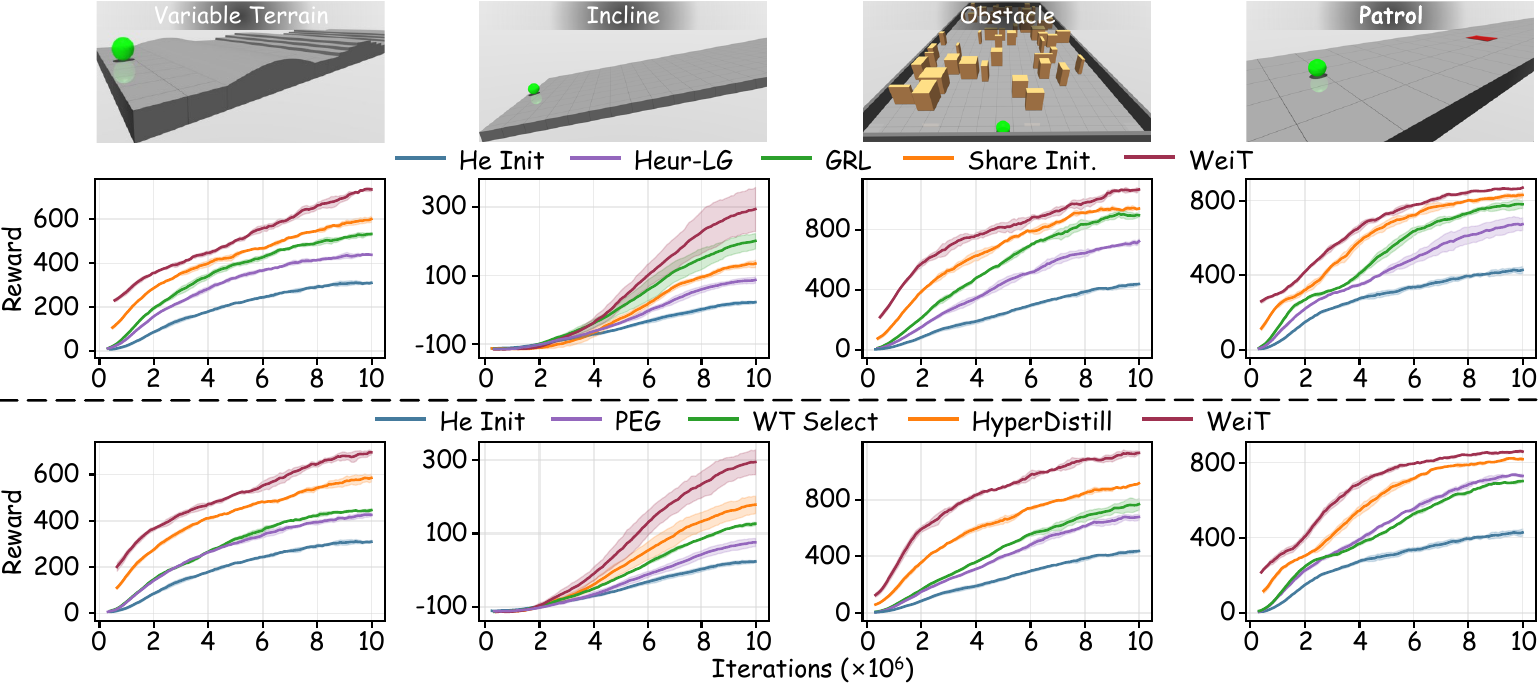}
    \vspace{-0.3in}
    \caption{\textbf{\textsc{(Embodied Control)}} 
    Performance of initialized models on downstream datasets with training morphologies using an $L_2H_2$ policy model. We further provide a visualization of the novel task environments, illustrating their variability relative to the training tasks.}
    \label{fig:control_downstream}
\end{figure*}

\subsection{Main Results on Embodied Control}
\label{sec:result_cont}

\subsubsection{Experimental Setup}
\begin{itemize}
\item \textbf{Basic Settings.}
We study the universal morphology control task within the UNIMAL design space~\cite{gupta2021embodied}, following MetaMorph~\cite{guptametamorph} with 100 training and 100 novel morphologies. 
The Morphology-Aware Transformer~\cite{guptametamorph} is adopted as the backbone, using $L_6H_4$ for constraint-based pre-training of WeiT on Flat Terrain (FT). 
Pre-training is conducted for 30 epochs with a batch size of 5120 and a learning rate of $1 \times 10^{-3}$, using Adam optimizer with a cosine learning rate scheduler.

\item \textbf{Evaluation.}
For multi-scale initialization, model depth is varied from $L_1$ to $L_{10}$ and width from $H_1$ to $H_8$. Performance is measured by the accumulated reward per episode. 
Transferability is further assessed on diverse novel tasks, including Variable Terrain (VT), Incline, Obstacle, and Patrol (see App.~\ref{app:rl_reward} for details).
\end{itemize}

\subsubsection{Initialization across Model Scales}
Unlike supervised vision tasks, embodied control requires learning dynamics-sensitive policies, where initialization critically affects exploration, stability, and convergence (see Fig.~\ref{fig:control_scale}). 
WeiT’s weight templates encode structured, reusable knowledge that is largely size-agnostic, enabling consistent performance across both small and large policy networks. 
This flexibility allows smaller models to achieve competitive rewards despite limited capacity (e.g., $L_1H_1$), while enabling larger models to converge faster and attain higher cumulative rewards. 
In contrast, comparative initialization methods, such as GRL~\cite{feng2023genes} and HyperDistill~\cite{xiong2024distilling}, often struggle to balance capacity and stability, resulting in suboptimal exploration and slower convergence.

Moreover, WeiT’s structured initialization reduces sensitivity to morphology variations, enabling robust transfer across both training and novel agent morphologies. 
These properties indicate that WeiT not only scales effectively across model sizes but also facilitates data-efficient and stable policy learning in complex reinforcement learning environments.

\subsubsection{Transferability to Downstream Datasets}
In \textsc{Embodied Control}, downstream tasks often exhibit substantial variation in environment dynamics and terrain complexity (see Fig.~\ref{fig:control_downstream}), making effective knowledge transfer crucial for data-efficient learning. 
WeiT addresses this by leveraging size-agnostic weight templates to encode reusable priors over control dynamics and morphology-aware behaviors, thereby enabling rapid adaptation to unseen tasks with minimal additional training. 
Notably, WeiT-initialized policies achieve stronger initial performance, indicating that the transferred priors facilitate early-stage learning and consequently lead to higher cumulative rewards and faster convergence.

\begin{table}
    \centering
    \setlength{\tabcolsep}{0.5 mm} 
    \caption{Performance on Initializing Convolution-based Models. We extend WeiT to ConvNeXt-v2 and evaluate its effectiveness on Image Classification.}
    \vspace{-0.07in}
    \resizebox{\linewidth}{!}{
        \begin{tabular}{@{}lccccc@{}}
        \toprule[1pt]
        & atto-$L_4$ & femto-$L_6$ & pico-$L_9$ & nano-$L_{12}$ & tiny-$L_{15}$    \\
        \cmidrule{2-6}
        \multirow{2}{*}{\textsc{Param.} / \textsc{FLOPs}} 
        & \cellcolor{gray!15}{1.7M} 
        & \cellcolor{gray!15}{3.0M} 
        & \cellcolor{gray!15}{7.4M} 
        & \cellcolor{gray!15}{13.9M} 
        & \cellcolor{gray!15}{25.0M}  \\
        & \cellcolor{gray!15}{0.4G} 
        & \cellcolor{gray!15}{0.8G} 
        & \cellcolor{gray!15}{2.1G} 
        & \cellcolor{gray!15}{4.2G} 
        & \cellcolor{gray!15}{7.5G}  \\
        \toprule[1pt]
        WT-Select~\cite{xu2023initializing}
            & 55.9 & 63.1 & 69.2 & 72.2 & 73.9 \\
        Iso. Pruning~\cite{fang2024isomorphic}
            & 50.6 & 57.6 & 65.3 & 70.0 & 72.9 \\
        \cellcolor{blue!12}{WeiT}
               & \cellcolor{blue!12}{\textbf{57.2}} 
               & \cellcolor{blue!12}{\textbf{65.4}} 
               & \cellcolor{blue!12}{\textbf{71.7}} 
               & \cellcolor{blue!12}{\textbf{74.9}} 
               & \cellcolor{blue!12}{\textbf{76.6}}  \\
                   & \textcolor{mygreen}{$\uparrow$1.3}
                   & \textcolor{mygreen}{$\uparrow$2.3} 
                   & \textcolor{mygreen}{$\uparrow$2.6}
                   & \textcolor{mygreen}{$\uparrow$2.7} 
                   & \textcolor{mygreen}{$\uparrow$2.7}  \\
        \bottomrule[1pt]
        \end{tabular}   
        }
\label{tab:cnn}
\vspace{-0.1in}
\end{table}

\begin{figure*}[!t]
    \centering
    \includegraphics[width=\linewidth]{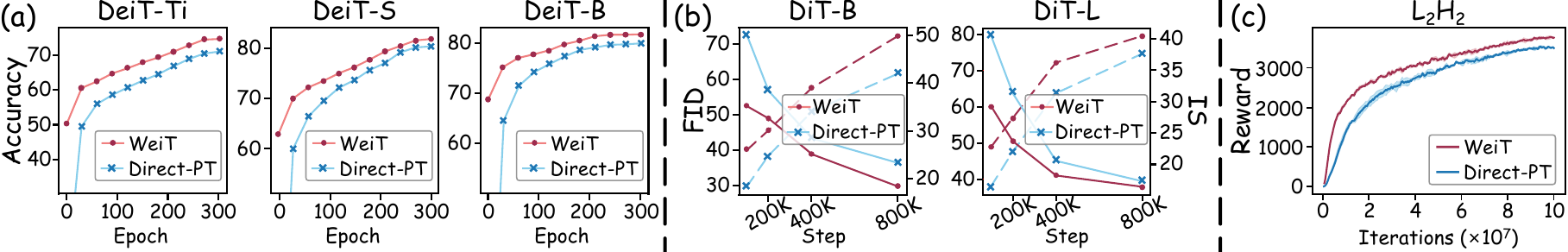}
    \vspace{-0.25in}
    \caption{Performance under Extended Training after Initialization. Full training is conducted for both directly pre-trained models (i.e., Direct PT) and WeiT-initialized models across \textsc{Image Classification}, \textsc{Image Generation}, and \textsc{Embodied Control}, with the number of training updates extended to 300 epochs, 800K steps, and $1\times10^8$ iterations, respectively.}
    \label{fig:full_train}
    \vspace{-0.15in}
\end{figure*}

\subsection{Performance on Convolution-based Architectures}
\label{sec:cnn}
To assess the generality of \textit{Constraint-based Pre-training} beyond Transformer-based architectures, we extend WeiT to Convolution-based models, focusing on ConvNeXt-v2~\cite{woo2023convnext}—a modern hierarchical convolutional backbone. 
By aggregating parameters of Convolution Kernels analogously to Transformers (Eq.~\eqref{eq:agg}), the proposed constraint is directly applicable (Eq.~\eqref{equ:kro}), with weight templates pre-trained on ImageNet-1K using a similar constraint-based procedure (see App.~\ref{app:cnn}).

Experimental results on \textsc{Image Classification} demonstrate that WeiT-initialized ConvNeXt consistently outperforms standard rule- and pruning-based initialization across variable model sizes (see Table~\ref{tab:cnn}), indicating that the proposed constraint-based weight templates capture transferable structural priors that remain effective under convolutional parameterizations and can be preserved and reused to support scalable initialization beyond Transformer-based architectures.

\subsection{Performance over Extended Training}
\label{sec:full_train}
To analyze the persistence of structural priors encapsulated in the weight templates, we examine training trajectories over extended horizons.
As shown in Fig.~\ref{fig:full_train}, WeiT-initialized models consistently converge faster and achieve higher final performance than training from scratch, indicating that the benefits of constraint-based pre-training persist beyond the early optimization stage and continue to influence the overall training process.
The sustained performance gap further suggests that the encapsulated priors act as enduring inductive structures rather than transient initialization effects, thereby making early-stage metrics (Table~\ref{tab:init_class}, Table~\ref{tab:init_gene}, and Fig.~\ref{fig:control_scale}) reliable indicators of long-term performance.

\subsection{Ablation and Analysis}
\label{sec:ablation}
\begin{table*}
    \centering
    \setlength{\tabcolsep}{1.2 mm} 
    \caption{Performance of Scale-up Initialization for Larger Models. Here, ``---'' indicates that the method does not support extension to the corresponding task, while ``NaN'' denotes cases where post-initialization training is unstable, resulting in undefined values.}
    \vspace{-0.1in}
    \resizebox{\linewidth}{!}{
        \begin{tabular}{@{}lcccc|cc||ccc|cc||cc|cc@{}}
        \toprule[1pt]
        & & \multicolumn{5}{c}{\textsc{Image Classification}}
        & \multicolumn{5}{c}{\textsc{Image Generation}}
        & \multicolumn{4}{c}{\textsc{Embodied Control}} \\
        \cmidrule(lr){3-7}
        \cmidrule(lr){8-12}
        \cmidrule(l){13-16}
        
        \multirow{2}{*}{\textsc{\textbf{\makecell{Depth / Width \\Var.}}}} & \multirow{2}{*}{\textbf{\makecell{w/ Train \\ Scaler}}} & \multicolumn{3}{c}{$H_6$} 
        & \multicolumn{2}{c||}{$L_6$} 
        & \multicolumn{3}{c}{$H_{12}$} 
        & \multicolumn{2}{c||}{$L_{12}$} 
        & \multicolumn{2}{c}{$H_{2}$} 
        & \multicolumn{2}{c}{$L_{2}$} \\
        \cmidrule(lr){3-5}
        \cmidrule(lr){6-7}
        \cmidrule(lr){8-10}
        \cmidrule(lr){11-12}
        \cmidrule(lr){13-14}
        \cmidrule(l){15-16}
        & & $L_{16}$ & $L_{20}$ & $L_{24}$ & $H_{18}$ & $H_{24}$
        & $L_{18}$ & $L_{21}$ & $L_{24}$ & $H_{24}$ & $H_{32}$  & $L_{8}$ & $L_{10}$ & $H_{6}$ & $H_{8}$ \\
        \cmidrule(r){1-5}
        \cmidrule(lr){6-7}
        \cmidrule(lr){8-10}
        \cmidrule(lr){11-12}
        \cmidrule(lr){13-14}
        \cmidrule(l){15-16}
        \multirow{2}{*}{\textsc{Param.} / \textsc{FLOPs}} & & \cellcolor{gray!15}{29.3M} 
        & \cellcolor{gray!15}{36.4M} 
        & \cellcolor{gray!15}{43.6M} 
        & \cellcolor{gray!15}{98.0M} 
        & \cellcolor{gray!15}{173.1M}
        & \cellcolor{gray!15}{194.1M} 
        & \cellcolor{gray!15}{226.0M} 
        & \cellcolor{gray!15}{257.9M} 
        & \cellcolor{gray!15}{519.0M} 
        & \cellcolor{gray!15}{921.6M}
        & \cellcolor{gray!15}{5.3M}
        & \cellcolor{gray!15}{6.6M}
        & \cellcolor{gray!15}{11.9M}
        & \cellcolor{gray!15}{21.1M}\\
        
        & & \cellcolor{gray!15}{12.0G} 
        & \cellcolor{gray!15}{15.0G} 
        & \cellcolor{gray!15}{18.0G} 
        & \cellcolor{gray!15}{38.8G} 
        & \cellcolor{gray!15}{68.5G} 
        & \cellcolor{gray!15}{65.4G} 
        & \cellcolor{gray!15}{76.3G} 
        & \cellcolor{gray!15}{87.2G} 
        & \cellcolor{gray!15}{174.4G} 
        & \cellcolor{gray!15}{310.0G}
        & \cellcolor{gray!15}{0.1G}
        & \cellcolor{gray!15}{0.1G}
        & \cellcolor{gray!15}{0.2G}
        & \cellcolor{gray!15}{0.3G}\\
        \toprule[1pt]
        He Init.~\cite{chen2021empirical} & 
                & 56.12 & 57.62 & 57.97
                & 45.62 & 30.16 
                & 67.73 & 63.88 & 59.66
                & 57.07 & 55.53 
                & 1568 & 1646
                & 1464 & 1341 \\
        LiGO~\cite{wang2023learning} & 
                & 75.00 & 76.22 & 76.47
                & 69.56 & 71.56 
                & 41.87 & 40.69 & 38.35
                & 43.82 & 46.18
                & --- & --- & --- & --- \\
        BoT~\cite{shen2026unified} & 
                & 75.86 & 76.29 & 76.57
                & 66.81 & 75.52 
                & 44.32 & 41.06 & 40.10
                & 63.98 & NaN
                & 2277 & 2344
                & 2131 & 1858 \\
        \midrule
        WeiT\waveicon & \ding{56}
             & 75.29 & 75.56 & 75.64
             & 72.31 & 74.99
             & 38.69 & 36.94 & 36.57
             & 49.32 & 46.70
             & 2649 & 2690
             & 2344 & 2319 \\
        WeiT\waveicon & \ding{51}
             & 75.47 & 75.62 & 75.74
             & 74.54 & 78.29
             & 37.19 & 36.25 & 35.12
             & 45.49 & 41.40
             & 2729 & 2785
             & 2643 & 2743 \\
        \cellcolor{blue!12}{WeiT} & \cellcolor{blue!12}{\ding{51}}
             & \cellcolor{blue!12}{\textbf{77.37}} & \cellcolor{blue!12}{\textbf{77.59}} & \cellcolor{blue!12}{\textbf{77.80}} & \cellcolor{blue!12}{\textbf{78.63}} & \cellcolor{blue!12}{\textbf{79.42}} 
             & \cellcolor{blue!12}{\textbf{35.77}} & \cellcolor{blue!12}{\textbf{35.21}} & \cellcolor{blue!12}{\textbf{32.74}} & \cellcolor{blue!12}{\textbf{40.12}} & \cellcolor{blue!12}{\textbf{37.11}} & \cellcolor{blue!12}{\textbf{2822}}
             & \cellcolor{blue!12}{\textbf{2866}}
             & \cellcolor{blue!12}{\textbf{2785}}
             & \cellcolor{blue!12}{\textbf{2852}}\\
            & & \textcolor{mygreen}{$\uparrow$1.51} 
            & \textcolor{mygreen}{$\uparrow$1.30}
            & \textcolor{mygreen}{$\uparrow$1.23} 
            & \textcolor{mygreen}{$\uparrow$4.09}
            & \textcolor{mygreen}{$\uparrow$1.13} 
            & \textcolor{mygreen}{$\downarrow$1.42} 
            & \textcolor{mygreen}{$\downarrow$1.04}
            & \textcolor{mygreen}{$\downarrow$2.38} 
            & \textcolor{mygreen}{$\downarrow$5.37}
            & \textcolor{mygreen}{$\downarrow$4.29}
            & \textcolor{mygreen}{$\uparrow$93}
            & \textcolor{mygreen}{$\uparrow$81}
            & \textcolor{mygreen}{$\uparrow$142}
            & \textcolor{mygreen}{$\uparrow$109}\\
        \bottomrule[1pt]
        
        \end{tabular}   
        }
    \label{tab:larger}
    \vspace{-0.15in}
\end{table*}

\begin{table}
    \centering
    \setlength{\tabcolsep}{0.8 mm} 
    \caption{Ablation Study on Constraint Types in Constraint-based Pre-training for Variable-sized Model Initialization.}
    \vspace{-0.1in}
    \resizebox{\linewidth}{!}{
        \begin{tabular}{@{}lcc||cc||cc@{}}
        \toprule[1pt]
        & \multicolumn{2}{c}{\textsc{\makecell{Image\\Classification}}} & \multicolumn{2}{c}{\textsc{\makecell{Image\\Generation}}} & \multicolumn{2}{c}{\textsc{\makecell{Embodied\\Control}}} \\
        \cmidrule(lr){2-3}
        \cmidrule(lr){4-5}
        \cmidrule(l){6-7}
        & $L_4 H_3$ & $L_8 H_6$ & $L_6 H_{12}$ & $L_8 H_{16}$ & $L_1 H_2$ & $L_2 H_3$ \\
        \cmidrule(r){1-3}
        \cmidrule(lr){4-5}
        \cmidrule(l){6-7}
        w/o Constraits
                & 36.33 & 60.15 & 74.97 & 69.57 & 1400 & 2210  \\
        \midrule
        Identity
                & 55.16 & 71.65 & 59.03 & 43.76 & 1691 & 2411  \\
        Linear
                & 55.00 & 72.14 & 54.97 & 44.32 & 1846 & 2593  \\
        SVD
                 & 57.88 & 73.76 & 51.58 & 41.24 & 1813 & 2475  \\
        \midrule
        \cellcolor{blue!12}{Kronecker}
                & \cellcolor{blue!12}{\textbf{58.64}} & \cellcolor{blue!12}{\textbf{74.06}} & \cellcolor{blue!12}{\textbf{48.40}} & \cellcolor{blue!12}{\textbf{29.72}} & \cellcolor{blue!12}{\textbf{1963}} & \cellcolor{blue!12}{\textbf{2699}} \\
        \bottomrule[1pt]
        \end{tabular}   
        }
    \label{tab:constait}
    \vspace{-0.13in}
\end{table}

\begin{table}
    \centering
    \setlength{\tabcolsep}{0.8 mm} 
    \caption{Ablation Study on Template Scaling Mechanism for Variable-sized Model Initialization.}
    \vspace{-0.1in}
    \resizebox{\linewidth}{!}{
        \begin{tabular}{@{}lccc||cc||cc@{}}
        \toprule[1pt]
        & \multirow{3}{*}{\textbf{\makecell{w/ Template\\Scaling}}} & \multicolumn{2}{c}{\textsc{\makecell{Image\\Classification}}} & \multicolumn{2}{c}{\textsc{\makecell{Image\\Generation}}} & \multicolumn{2}{c}{\textsc{\makecell{Embodied\\Control}}} \\
        \cmidrule(lr){3-4}
        \cmidrule(lr){5-6}
        \cmidrule(l){7-8}
        & & $L_6 H_4$ & $L_{12} H_8$ & $L_6 H_9$ & $L_{12} H_{15}$ & $L_1 H_2$ & $L_2 H_3$ \\
        \cmidrule(r){1-4}
        \cmidrule(lr){5-6}
        \cmidrule(l){7-8}
        WT-Select & 
                & 49.26 & 71.80 & 79.64 & 37.14 & 1400 & 2210  \\
        \midrule
        WeiT & \ding{56}
                & 59.48 & 77.75 & 58.96 & 30.86 & 1583 & 2388  \\
        \cellcolor{blue!12}{WeiT} & \cellcolor{blue!12}{\ding{51}}
                & \cellcolor{blue!12}{\textbf{64.20}} & \cellcolor{blue!12}{\textbf{79.34}}
                & \cellcolor{blue!12}{\textbf{48.94}} & \cellcolor{blue!12}{\textbf{27.11}}& \cellcolor{blue!12}{\textbf{1980}} & \cellcolor{blue!12}{\textbf{2714}} \\
        \bottomrule[1pt]
        \end{tabular}   
        }
    \label{tab:abl}
    \vspace{-0.1in}
\end{table}

\begin{figure}[tb]
  \centering
  \includegraphics[width=\linewidth]{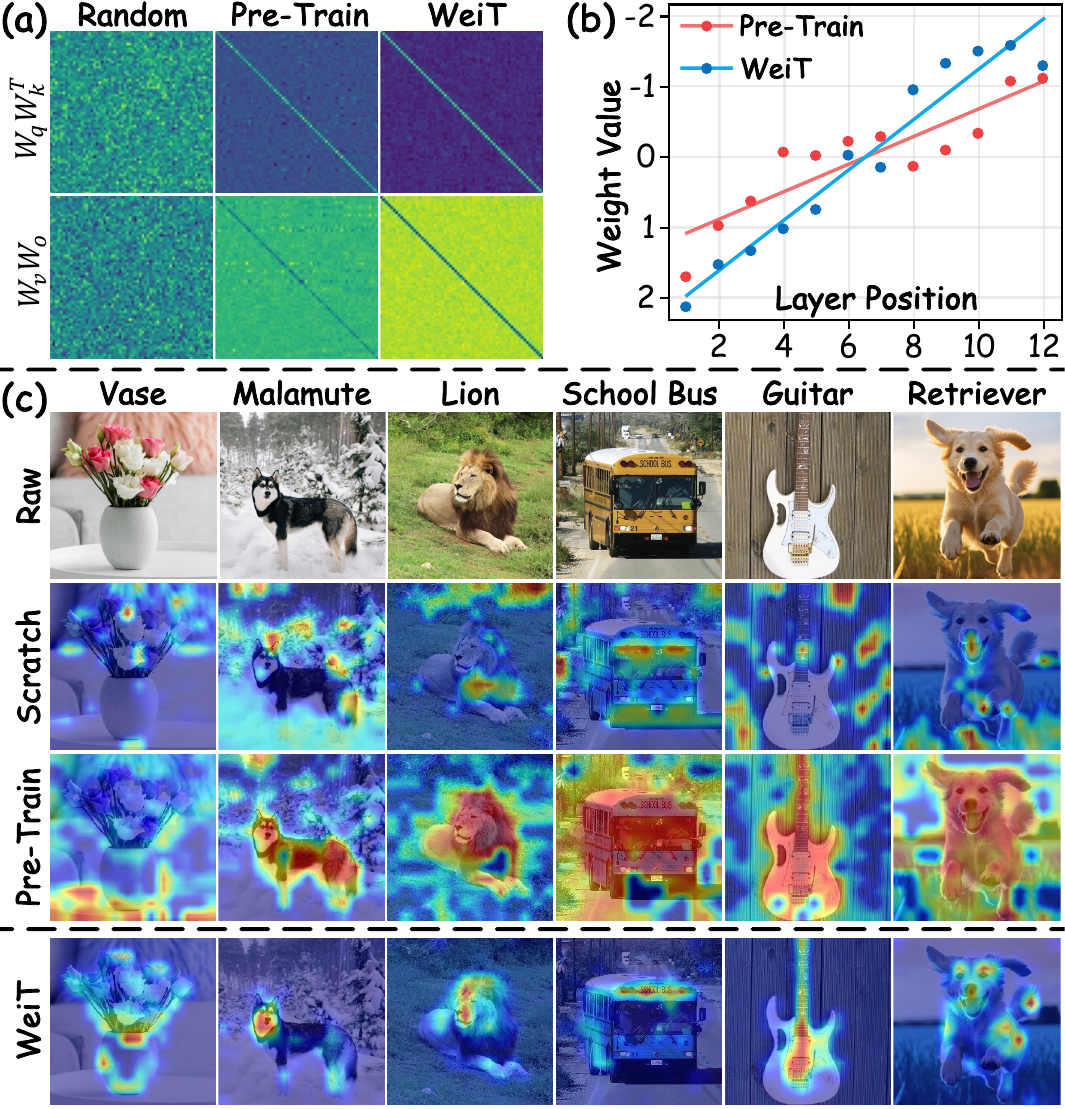}
  \vspace{-0.3in}
  \caption{Visualization of Knowledge Encapsulated in Weight Templates.
  (a)~Knowledge captured in self-attention layers, with weight templates inheriting the diagonal property characteristic of pre-trained ViTs~\cite{trockman2023mimetic}.
  (b)~Relationship between layer position and corresponding parameter values after PCA~\cite{greenacre2022principal}, with models initialized by WeiT reflecting the approximately linear patterns observed in pre-trained models~\cite{xia2024transformer}.
  (c)~Attention visualized using CAM~\cite{selvaraju2017grad}, with models initialized by WeiT capturing the core attention patterns directly after initialization~\cite{feng2024transferring}.}
  \label{fig:structure}
  \vspace{-0.2in}
\end{figure}

\subsubsection{Analysis on Constraint Types}
\label{sec:constr_type}
Constraint-based pre-training introduces structural constraints to regularize optimization and facilitate the extraction of size-agnostic knowledge.
Table~\ref{tab:constait} presents WeiT’s performance under different constraints, showing that all structured constraints contribute positively, whereas unconstrained pre-training produces tightly coupled weights that generalize poorly to unseen scales.

Moreover, Identity- and Linear-based constraints restrict weight templates to simple layer-wise reuse, thereby restricting both expressiveness and inter-layer specificity. 
SVD-based constraints partially address this by enabling depth-wise decomposition, enhancing inter-layer specialization, yet they remain relatively inflexible in accommodating width variations.
In contrast, Kronecker-based constraints disentangle knowledge along both depth and width, preserving layer-specific patterns while enabling flexible width-wise feature reuse, thereby yielding expressive, size-agnostic representations that generalize robustly across models of varying sizes.

\subsubsection{Analysis on Low-Rank Bottleneck}
\label{sec:low_rank}
During constraint-based pre-training, a low-rank bottleneck is imposed on the weight templates to maintain a substantially lower parameter count than the corresponding pre-trained model.
This design encourages template reuse, enabling scalable initialization beyond the original model size while preserving expressive capacity.
As shown in Table~\ref{tab:larger}, WeiT-initialized models not only outperform direct pre-training when scaling to larger sizes but also surpass methods specifically designed for model expansion, such as LiGO~\cite{wang2023learning} and BoT~\cite{shen2026unified}.

\subsubsection{Ablation on the Training of Weight Scalers}
To evaluate the effectiveness of the training of lightweight weight scalers, we conduct an ablation study on their role in adapting weight templates to target model sizes. 
Unlike distillation-based methods that require retraining full model parameters, WeiT optimizes only a small set of weight scalers to reconstruct size-specific weights from shared templates. 
As shown in Table~\ref{tab:larger}, this lightweight adaptation incurs negligible overhead (a minute-level of wall-clock time) while achieving superior performance, indicating that weight scalers effectively capture concatenation and weighted aggregation rules for adapting templates across varying depths and widths.

\subsubsection{Ablation on the Template Scaling Mechanism}
The Template Scaling Mechanism applies structured dropout to the weight templates during constraint-based pre-training, encouraging the reorganization of size-agnostic knowledge in the weight templates along the width dimension.
As shown in Table~\ref{tab:abl}, this mechanism facilitates effective generalization and stable initialization across width-variant models. 
Without it, templates may lose critical information when adapting to reduced widths, resulting in degraded performance.

\subsection{Visualization of Knowledge in Weight Templates}
\label{sec:visual}
\subsubsection{Visualization of Structured Knowledge}
Prior works, such as Mimetic Init.~\cite{trockman2023mimetic} and TLEG~\cite{xia2024transformer}, reveal diagonal patterns in self-attention layers and linear correlations across layers, respectively. However, these findings are specific to pre-trained ViTs and require manual preservation of structured knowledge during initialization. 
Remarkably, as shown in Fig.~\ref{fig:structure}a,b, WeiT autonomously captures such structural patterns within its weight templates, without any manual intervention. 
As a result, models initialized from these templates inherently preserve such characteristic structures in their parameter matrices.

\subsubsection{Visualization of Common Knowledge}
We further demonstrate in \textsc{Image Classification} that weight templates guide models to focus more on common local features after initialization.
As shown in Fig.~\ref{fig:structure}c, random initialization shows scattered and widespread attention, whereas pre-trained models transfer a full set of previously learned knowledge, leading to broader attention that often includes irrelevant regions, such as image background.
In contrast, WeiT focus more on local features (i.e., smaller red attention regions), demonstrating superior localization and a cleaner focus (i.e., removing attention on image background), thereby enhancing classification performance.

\section{Conclusion}
We introduce constraint-based pre-training, a novel paradigm for pre-training models that can flexibly initialize variable-sized downstream models.
Within this paradigm, we propose WeiT, which leverages Kronecker-based constraints to encapsulate size-agnostic knowledge into weight templates, complemented by lightweight, size-specific weight scalers for efficient initialization across diverse model scales.
WeiT achieves superior performance in both depth and width scaling across multiple tasks, including \textsc{Image Classification}, \textsc{Image Generation}, and \textsc{Embodied Control}. 
Its effectiveness further generalizes across both Transformer-based and Convolution-based architectures, while consistently exhibiting faster convergence and superior performance under full training.

\section*{Acknowledgments}
This research was supported by the Jiangsu Science Foundation (BG2024036, BK20243012), the National Natural Science Foundation of China (625B2045, 62125602, U24A20324, 92464301, 62306073), the New Cornerstone Science Foundation through the XPLORER PRIZE, the Fundamental Research Funds for the Central Universities (2242025K30024), and SEU Innovation Capability Enhancement Plan for Doctoral Students (CXJH\_SEU 26023).

\bibliographystyle{IEEEtran}
\bibliography{reference.bib}

\section{Biography Section}
\begin{IEEEbiography}[{\includegraphics[width=1in,height=1.25in,clip,keepaspectratio]{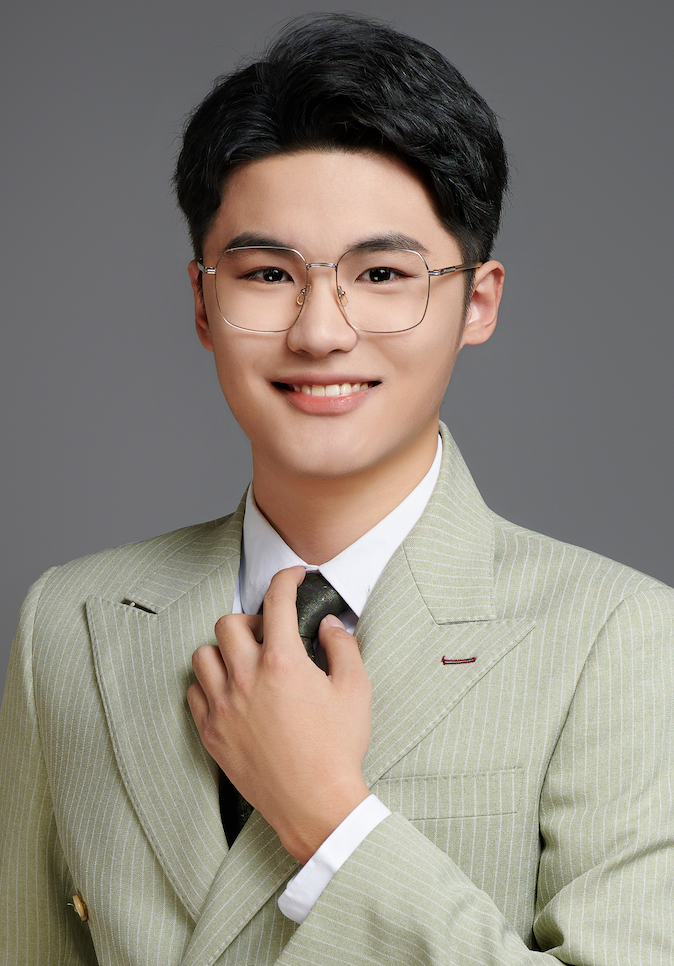}}]
{Fu Feng}
received the B.Sc. (Hons.) degree in artificial intelligence from Chien-Shiung Wu College, Southeast University, Nanjing, China, in 2023. 
He is currently pursuing the Ph.D. degree in the School of Computer Science and Engineering, Southeast University. 
His research interests include machine learning, intelligent agents and creative generation.
\end{IEEEbiography}

\begin{IEEEbiography}[{\includegraphics[width=1in,height=1.25in,clip,keepaspectratio]{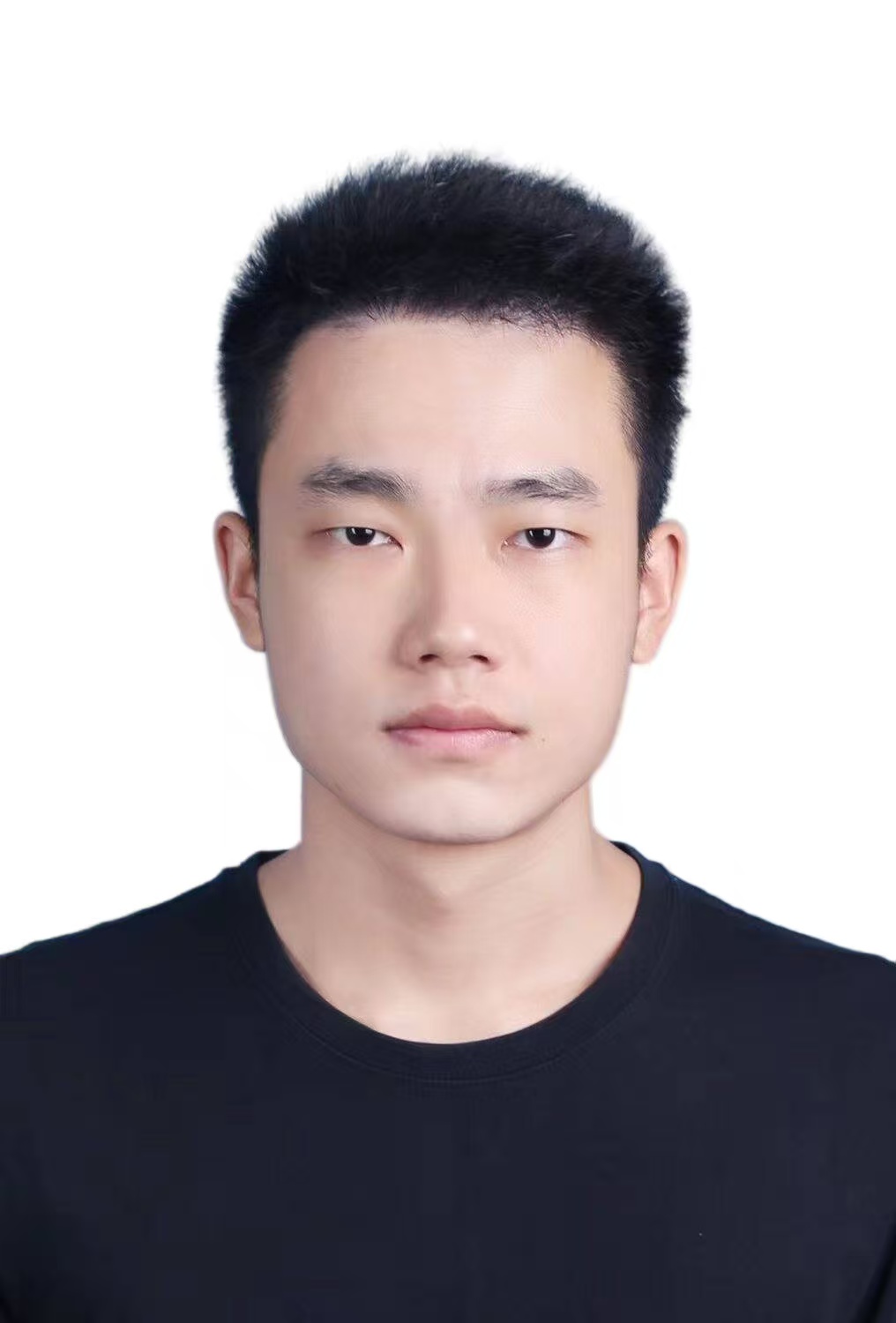}}]
{Yucheng Xie}
received the B.Sc. degree in computer science from Ocean University of China, Qingdao, China, in 2023. 
He is currently pursuing the Ph.D. degree in the School of Computer Science and Engineering, Southeast University. 
His research interests include machine learning and computer vision.
\end{IEEEbiography}

\begin{IEEEbiography}[{\includegraphics[width=1in,height=1.25in,clip,keepaspectratio]{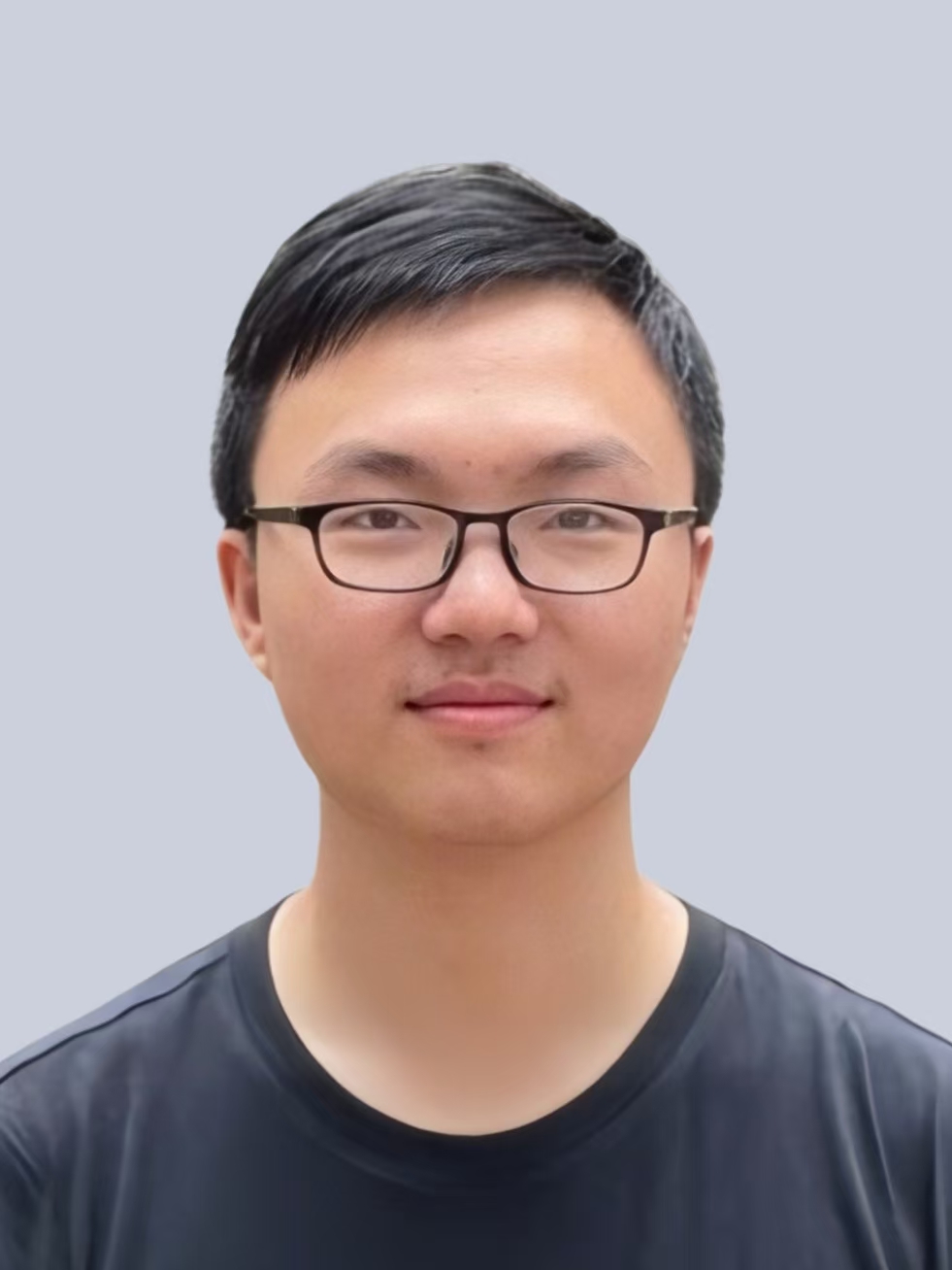}}]
{Ruixiao Shi}
received the B.Sc. degree in artificial intelligence from Southeast University, Nanjing, China, in 2025. 
He is currently pursuing the M.Sc degree in the School of Computer Science and Engineering, Southeast University. 
His research interests include machine learning and intelligent agents.
\end{IEEEbiography}

\begin{IEEEbiography}[{\includegraphics[width=1in,height=1.25in,clip,keepaspectratio]{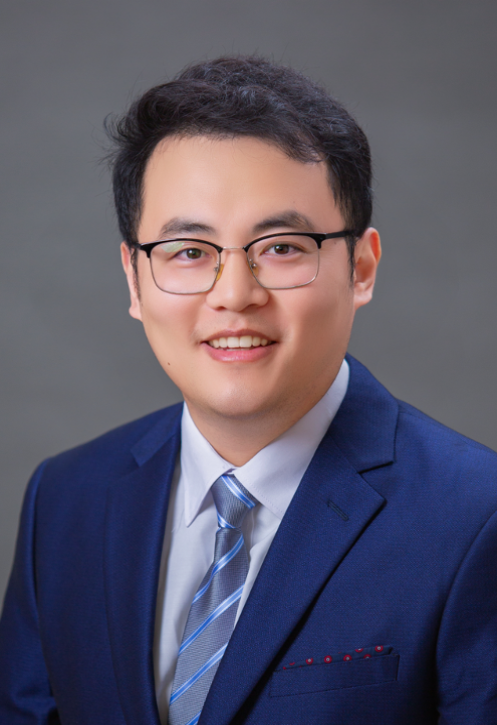}}]
{Jing Wang}
received the B.Sc. degree in computer science from Suzhou University of Science and Technology, Suzhou, China, in 2013, and the M.Sc. degree in computer science from Northeastern University, Shenyang, China, in 2015, and the Ph.D. degree in software engineering from Southeast University, Nanjing, China, in 2021. 
He is currently an assistant professor  of the School of Computer Science and Engineering, Southeast University, Nanjing. His research interests include pattern recognition and machine learning.
\end{IEEEbiography}

\begin{IEEEbiography}[{\includegraphics[width=1in,height=1.25in,clip,keepaspectratio]{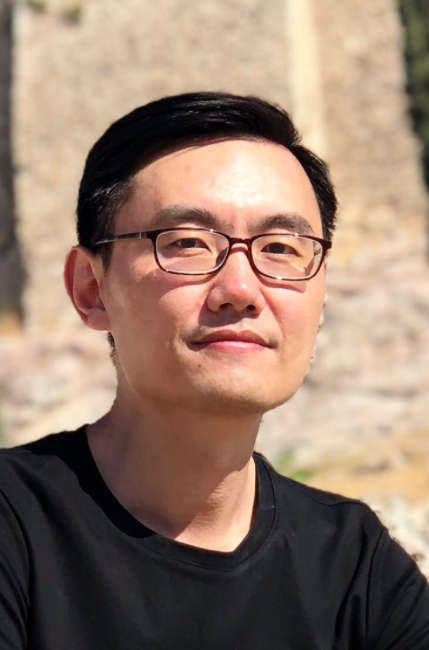}}]
{Xin Geng (Senior Member, IEEE)}
received the B.Sc. and M.Sc. degrees in computer science from Nanjing University, Nanjing, China, in 2001 and and 2004, respectively, and the Ph.D. degree in computer science from Deakin University, Geelong, VIC, Australia, in 2008. 

He is currently a chair professor of the School of Computer Science and Engineering, Southeast University, Nanjing. His research interests include machine learning, pattern recognition, and computer vision. He has published over 100 refereed articles in these areas, including those published in prestigious journals and top international conferences. 

Dr. Geng has been an Associate Editor of IEEE TRANSACTIONS ON MULTIMEDIA, Frontiers of Computer Science, and Mathematical Foundations of Computing, a Steering Committee Member of Pacific Rim International Conferences on Artificial Intelligence (PRICAI), a Program Committee Chair for conferences, such as PRICAI 2018 and  Vision And Learning SEminar (VALSE) 2013, the Area Chair for conferences, such as Computer Vision and Pattern Recognition (CVPR), ACM Multimedia, and Chinese Conference on Pattern Recognition (CCPR), and a Senior Program Committee Member for conferences, such as International Joint Conference on Artificial Intelligence (IJCAI), AAAI Conference on Artificial Intelligence (AAAI), and European Conference on Artificial Intelligence (ECAI). He is a Distinguished Fellow of International Engineering and Technology Institute. 
\end{IEEEbiography}

\clearpage
\appendices
\onecolumn
\renewcommand{\arraystretch}{1}

\section{Theoretical Insights of WeiT}
\label{sec:proof}
\subsection{Theoretical Guarantee of Expressivity under Kronecker-based Constraints}
\label{sec:theoretical_guarantee}

To address potential concerns regarding whether the imposed Kronecker-based constraints and the low-rank bottleneck restrict the hypothesis space too severely, we provide a theoretical guarantee. By interpreting the weight templates as a shared ``knowledge dictionary'' and the scalers as ``local combination coefficients'', we demonstrate that our constraint-based formulation retains the universal approximation capabilities of the unconstrained network, provided that the number of templates $N$ is appropriately bounded by the singular value decay of the target weights.

\begin{definition}[Parameter Rearrangement Operator $\mathcal{R}$]
Let the original unconstrained unified weight matrix be $\mathcal{W} \in \mathbb{R}^{L \times P}$. Given the weight templates $\mathcal{T}_i \in \mathbb{R}^{1 \times (r_1 \cdot r_2)}$ and scalers $\mathcal{S}_{\star,i} \in \mathbb{R}^{L \times B}$ (where $B = \frac{P}{r_1 \cdot r_2}$), we define the feature dimension span as $A = r_1 \cdot r_2$. We can conceptually partition $\mathcal{W}$ into $L \cdot B$ sub-blocks, each of size $1 \times A$. We define the rearrangement operator $\mathcal{R}: \mathbb{R}^{L \times P} \to \mathbb{R}^{(L \cdot B) \times A}$ such that the $k$-th row of the rearranged matrix $\tilde{\mathcal{W}} = \mathcal{R}(\mathcal{W})$ corresponds to the $k$-th $(1 \times A)$ sub-block of the original matrix $\mathcal{W}$.
\end{definition}

\begin{lemma}[Truncated Singular Value Bound of Kronecker Decomposition]
\label{lemma:svd_bound}
For any given unconstrained weight matrix $\mathcal{W} \in \mathbb{R}^{L \times P}$, the constrained weight matrix $\mathcal{W}_{\star} = \sum_{i=1}^{N} \mathcal{T}_{i} \otimes \mathcal{S}_{\star,i}$ is mathematically equivalent to a rank-$N$ matrix factorization in the rearranged space $\mathbb{R}^{(L \cdot B) \times A}$. Furthermore, the Frobenius norm of its minimum reconstruction error is strictly bounded by the truncated singular values of $\tilde{\mathcal{W}}$:
\begin{equation*}
    \min_{\mathcal{T}, \mathcal{S}} \|\mathcal{W} - \mathcal{W}_{\star}\|_F^2 = \sum_{j=N+1}^{\min(L \cdot B, A)} \sigma_j^2(\tilde{\mathcal{W}}),
\end{equation*}
where $\sigma_j(\tilde{\mathcal{W}})$ denotes the $j$-th largest singular value of the rearranged matrix $\tilde{\mathcal{W}}$.
\end{lemma}

\begin{proof}
Based on the properties of the Kronecker product, the rearranged form of the block matrix $\mathcal{T}_i \otimes \mathcal{S}_{\star,i}$ under the operator $\mathcal{R}$ maps the elements of the scaler $\mathcal{S}_{\star,i}$ to the corresponding rows, multiplying the shared template $\mathcal{T}_i$. This is exactly equivalent to the outer product of two vectors: $\mathbf{s}_i \mathbf{t}_i^\top$, where $\mathbf{s}_i = \text{vec}(\mathcal{S}_{\star,i}) \in \mathbb{R}^{(L \cdot B) \times 1}$ represents the local combination coefficients, and $\mathbf{t}_i^\top = \mathcal{T}_i \in \mathbb{R}^{1 \times A}$ represents the shared dictionary basis. Consequently, the rearranged constrained matrix can be formulated as:
\begin{equation*}
    \tilde{\mathcal{W}}_{\star} = \mathcal{R}(\mathcal{W}_{\star}) = \sum_{i=1}^{N} \mathbf{s}_i \mathbf{t}_i^\top.
\end{equation*}
This establishes $\tilde{\mathcal{W}}_{\star}$ as a matrix factorization of $\tilde{\mathcal{W}}$ with a maximum rank of $N$. 
According to the Eckart-Young-Mirsky theorem, the optimal rank-$N$ approximation error for $\tilde{\mathcal{W}}$ in terms of the Frobenius norm is given by the sum of squared singular values from index $N+1$ onwards. Since the Frobenius norm is invariant under the spatial rearrangement operator $\mathcal{R}$ (i.e., $\|\mathcal{W} - \mathcal{W}_{\star}\|_F = \|\tilde{\mathcal{W}} - \tilde{\mathcal{W}}_{\star}\|_F$), the lemma is proven.
\end{proof}

\begin{lemma}[Lipschitz Continuity of Network Output]
\label{lemma:lipschitz}
Assuming the base neural network $f(\mathbf{x}; \mathcal{W})$ employs Lipschitz continuous activation functions (e.g., GELU) and the input domain of $\mathbf{x}$ is bounded, there exists a constant $K > 0$ such that for any two weight matrices $\mathcal{W}_1$ and $\mathcal{W}_2$, the network output satisfies:
\begin{equation*}
    \|f(\mathbf{x}; \mathcal{W}_1) - f(\mathbf{x}; \mathcal{W}_2)\| \le K \|\mathcal{W}_1 - \mathcal{W}_2\|_F.
\end{equation*}
\end{lemma}

\begin{theorem}[Universal Approximation under WeiT Constraints]
\label{theorem:universal_approx}
Suppose there exists an ideal continuous function $f(\mathbf{x}) = f(\mathbf{x}; \mathcal{W})$ parameterized by a standard unconstrained Transformer for a target task. For any given approximation error tolerance $\epsilon > 0$, if the number of weight templates $N$ (the Kronecker rank) satisfies:
\begin{equation*}
    \sum_{j=N+1}^{\min(L \cdot B, A)} \sigma_j^2(\mathcal{R}(\mathcal{W})) \le \frac{\epsilon^2}{K^2},
\end{equation*}
then there guarantees to exist a set of weight templates $\mathcal{T}$ and scalers $\mathcal{S}$ such that the resulting WeiT-constrained network $f(\mathbf{x}; \mathcal{W}_{\star})$ satisfies:
\begin{equation*}
    \sup_{\mathbf{x}} \|f(\mathbf{x}; \mathcal{W}_{\star}) - f(\mathbf{x}; \mathcal{W})\| \le \epsilon.
\end{equation*}
\end{theorem}

\begin{proof}
Combining Lemma \ref{lemma:svd_bound} and Lemma \ref{lemma:lipschitz}, let the target ideal parameters be $\mathcal{W}$. By treating the constrained network matrix $\mathcal{W}_{\star}$ as the optimal rank-$N$ approximation of $\mathcal{W}$ in the rearranged space, we obtain the inequality bounded by the Lipschitz constant: $\|f(\mathbf{x}; \mathcal{W}) - f(\mathbf{x}; \mathcal{W}_{\star})\| \le K \|\mathcal{W} - \mathcal{W}_{\star}\|_F$. Substituting the minimum reconstruction error from Lemma \ref{lemma:svd_bound} into this inequality yields:
\begin{equation*}
    \sup_{\mathbf{x}} \|f(\mathbf{x}; \mathcal{W}_{\star}) - f(\mathbf{x}; \mathcal{W})\| \le K \sqrt{\sum_{j=N+1}^{\min(L \cdot B, A)} \sigma_j^2(\mathcal{R}(\mathcal{W}))}.
\end{equation*}
Given the condition that the sum of the squared truncated singular values is bounded by $\frac{\epsilon^2}{K^2}$, the overall output discrepancy of the constrained network is strictly bounded within $\epsilon$. This completes the proof.
\end{proof}

\subsection{Generalization Bounds for Parameter-Efficient Initialization}
\label{sec:generalization_bounds}

To formalize the efficiency of our parameter-efficient initialization, we provide a theoretical analysis based on statistical learning theory. We aim to demonstrate that by freezing the weight templates $\mathcal{T}$ and optimizing only the lightweight scalers $\mathcal{S}_t$, WeiT significantly restricts the hypothesis space, thereby yielding a tighter generalization bound compared to full fine-tuning, especially in data-scarce scenarios.

\begin{definition}[Hypothesis Spaces]
\label{def:hypothesis_space}
For a target downstream model $\mathcal{M}_t$, let the loss function be $\ell(f(\mathbf{x}; \mathcal{W}_t), y)$. Under standard full fine-tuning, the hypothesis space spans the entire unconstrained weight matrix $\mathcal{W}_t \in \mathbb{R}^{L_t \times P_t}$:
\begin{equation*}
    \mathcal{H}_{\text{Full-FT}} = \{ f(\mathbf{x}; \mathcal{W}_t) \mid \|\mathcal{W}_t\|_F \le C_W \},
\end{equation*}
where $C_W > 0$ bounds the norm of the weights. 

In contrast, under the proposed constraint-based initialization (i.e., WeiT), the size-agnostic weight templates $\mathcal{T}$ are fixed after pre-training. The trainable parameters are strictly confined to the lightweight weight scalers $\mathcal{S}_t = \{\mathcal{S}_{t, i}\}_{i=1}^{N}$, where each $\mathcal{S}_{t, i} \in \mathbb{R}^{L_t\times B_t }$ (where $B_t = \frac{P_t}{r_1 \cdot r_2}$). The hypothesis space is thus restricted to:
\begin{equation*}
    \mathcal{H}_{\text{WeiT}} = \left\{ f(\mathbf{x}; \mathcal{T} \otimes \mathcal{S}_t) \mid \|\mathcal{S}_t\|_F \le C_S \right\}.
\end{equation*}
Given the structural property of the Kronecker factorization, the large parameter volume is fundamentally absorbed by the spatial dimensions of the weight templates ($r_1 \cdot r_2$). Consequently, the trainable parameter dimension of the scalers is reduced to $N \cdot L_t \cdot B_t = L_t \cdot P_t \cdot \left(\frac{N}{r_1 \cdot r_2}\right)$. By configuring $N \ll r_1 \cdot r_2$, the trainable parameter space in $\mathcal{H}_{\text{WeiT}}$ becomes vastly smaller than that in $\mathcal{H}_{\text{Full-FT}}$, allowing for a substantially tighter constraint $C_S \ll C_W$.
\end{definition}

We utilize empirical Rademacher complexity to measure the capacity of the hypothesis space and bound the generalization gap on downstream datasets.

\begin{lemma}[Rademacher Complexity of WeiT]
\label{lemma:rademacher}
Assume the downstream task loss function $\ell$ is $L_\ell$-Lipschitz continuous with respect to the network output, and the input features are bounded by $\|\mathbf{x}\| \le R$. Given a downstream dataset of size $m$, the empirical Rademacher complexity of the WeiT hypothesis space is bounded by:
\begin{equation*}
    \hat{\mathfrak{R}}_m(\mathcal{H}_{\text{WeiT}}) \le \mathcal{O}\left( \frac{L_\ell R C_S \|\mathcal{T}\|_F}{\sqrt{m}} \right).
\end{equation*}
\end{lemma}

\begin{proof}
Based on the definition of the Kronecker product, the Frobenius norm of the reconstructed target weight matrix satisfies the sub-multiplicative property:
\begin{equation*}
    \|\mathcal{W}_t\|_F = \|\mathcal{T} \otimes \mathcal{S}_t\|_F \le \|\mathcal{T}\|_F \|\mathcal{S}_t\|_F.
\end{equation*}
Since $\mathcal{T}$ is pre-trained and frozen during initialization, $\|\mathcal{T}\|_F$ acts as a constant scaling factor. The complexity of the network is thus dictated exclusively by the capacity of $\mathcal{S}_t$, which is bounded by $C_S$. According to standard Rademacher complexity bounds for neural networks parameterized by weight matrices with bounded Frobenius norms, the complexity scales linearly with the norm bound of the trainable matrices, yielding the upper bound presented above.
\end{proof}

\begin{theorem}[Generalization Bound for Scalable Initialization]
\label{theorem:generalization}
For any downstream task with an empirical risk $\hat{R}(h)$ evaluated on $m$ training samples, let $R(h)$ denote the expected true risk. With probability at least $1-\delta$, for all hypotheses $h \in \mathcal{H}_{\text{WeiT}}$, the following generalization bound holds:
\begin{equation*}
    R(h) \le \hat{R}(h) + 2 \hat{\mathfrak{R}}_m(\mathcal{H}_{\text{WeiT}}) + 3 \sqrt{\frac{\log(2/\delta)}{2m}}.
\end{equation*}
Substituting the complexity bound from Lemma \ref{lemma:rademacher}, we obtain:
\begin{equation*}
    R(h) \le \hat{R}(h) + \mathcal{O}\left( \frac{L_\ell R C_S \|\mathcal{T}\|_F}{\sqrt{m}} \right) + 3 \sqrt{\frac{\log(2/\delta)}{2m}}.
\end{equation*}
\end{theorem}

Theorem \ref{theorem:generalization} theoretically supports the data-efficiency of WeiT. Although the templates $\mathcal{T}$ encapsulate massive pre-trained knowledge, freezing them prevents them from inflating the variance of the hypothesis space. Consequently, the generalization gap is strictly bounded by the small capacity $C_S$ of the scalers. This mathematical formulation explains why WeiT can achieve stable and high-performance initialization with merely a few hundred optimization steps, preventing the severe overfitting typically observed when directly fine-tuning large unconstrained models on scarce data.

\section{Additional Methodological Details}
\subsection{Details of the Constraint-based Pre-training Paradigm}
Algorithm~\ref{alg:algorithm} presents the pseudocode of the proposed constraint-based pre-training paradigm, which aims to encapsulate size-agnostic knowledge into weight templates.

\begin{algorithm}[h]
    \caption{Constraint-based Pre-training Paradigm}
    \label{alg:algorithm}
    \textbf{Input}: Training dataset $\mathcal{D}=\{(x_i, y_i)\}_{i=1}^{|\mathcal{D}|}$, Number of training epochs $N_{\text{ep}}$, Batch size $B$, Learning rate $\eta$, and Model to be pre-trained $\mathcal{M}_\star$ parameterized by $\Theta_\star$ \\
    \textbf{Output}: Size-agnostic Weight Templates $\mathcal{T}$ and Size-specific Weight Scalers $\mathcal{S}_\star$
    \begin{algorithmic}[1]
        \STATE Randomly initialize Weight Templates $\mathcal{T}$ and Weight Scalers $\mathcal{S}_\star$
        \FOR{$\text{ep} = 1$ to $N_{\text{ep}}$}
            \FOR{each mini-batch $\mathcal{B} = \{(x_i, y_i)\}_{i=1}^B$}
                \STATE Randomly generate a structured mask $M_{\mathcal{T}}$ and apply it to the weight templates $\mathcal{T}$ to obtain the masked templates $\tilde{\mathcal{T}}$ following Eq.~\eqref{equ:drop}:
                \[
                    \tilde{\mathcal{T}} = M_\mathcal{T} \odot \mathcal{T}
                \]
                \vspace{-0.13in}
                \STATE Construct the unified weight matrix $\mathcal{W}_\star$ from the weight templates $\mathcal{T}$ and scalers $\mathcal{S}_\star$ via the Kronecker product, following Eq.~\eqref{equ:kro_mask}:
                \[
                    \mathcal{W}_\star \;=\; \tilde{\mathcal{T}} \otimes \mathcal{S}_\star
                \]
                \vspace{-0.13in}
                \STATE Decompose $\mathcal{W}_\star$ and map it back to the original parameter space to reconstruct $\Theta_\star$ via parameter replacement, corresponding to the inverse operation of Eq.~\eqref{eq:agg}.
                \STATE For each input $x_i$, obtain the model prediction $\hat{y}_i = \mathcal{M}_\star(x_i)$
                \STATE Compute the batch loss $\mathcal{L}_{\text{batch}} = \frac{1}{B} \sum_{i=1}^{B} \mathcal{L}(\hat{y}_i, y_i)$, where $\mathcal{L} \in \{\mathcal{L}_{\textsc{Cls}}, \mathcal{L}_{\textsc{Gen}}, \mathcal{L}_{\textsc{Ctl}}\}$ corresponds to \textsc{Image Classification} (Eq.~\eqref{equ:loss_cls}), \textsc{Image Generation} (Eq.~\eqref{equ:loss_gen}), and \textsc{Embodied Control} (Eq.~\eqref{equ:loss_ctl}), respectively.
                \STATE Backpropagate the batch loss $\mathcal{L}_{\text{batch}}$ to compute gradients with respect to the weight templates $\mathcal{T}$ and scalers $\mathcal{S}_\star$:
                \[
                    \nabla_{\mathcal{T}} \mathcal{L}_{\text{batch}}, \quad \nabla_{\mathcal{S}_\star} \mathcal{L}_{\text{batch}}.
                \]
                \vspace{-0.13in}
                \STATE Update the weight templates $\mathcal{T}$ and scalers $\mathcal{S}_\star$ via gradient descent: 
                \[
                    \mathcal{T} \leftarrow \mathcal{T} - \eta \nabla_{\mathcal{T}} \mathcal{L}_{\text{batch}}, \quad
                    \mathcal{S}_\star \leftarrow \mathcal{S}_\star - \eta \nabla_{\mathcal{S}_\star} \mathcal{L}_{\text{batch}}.
                \]
            \ENDFOR
        \ENDFOR
    \end{algorithmic}
\end{algorithm}

\subsection{Extension to Convolution-based Architectures}
\label{app:cnn}
\subsubsection{Preliminaries on ConvNeXt-v2}
ConvNeXt-v2~\cite{woo2023convnext} is a modern hierarchical Convolution-based architecture with a stage-wise design, where feature maps are progressively downsampled. Given an input image $x \in \mathbb{R}^{H \times W \times 3}$, the network extracts multi-scale representations across $S$ stages.

Formally, ConvNeXt-v2 can be viewed as a composition of stage-wise transformations parameterized by the full set of network parameters $\Theta = \{\Theta_s\}_{s=1}^{S}$:
\begin{equation*}
    Y = \mathcal{F}_{\Theta}(x),
\end{equation*}
where $\mathcal{F}_{\Theta}(\cdot)$ denotes the overall forward mapping of the network.

Specifically, each stage transformation $\Theta_s$ is instantiated by a sequence of learnable weight operators, including depthwise convolution and channel-mixing MLP components, with the full parameter set defined as $\Theta = \{W^{\text{dw}}_s, W_{1,s}, W_{2,s}\}_{s=1}^{S}$.

At stage $s$, the spatial resolution is reduced by a factor of $2$ relative to the previous stage, yielding
\begin{equation*}
    (H_s, W_s) = \left(\frac{H}{2^s}, \frac{W}{2^s}\right), \quad s = 1, \dots, S,
\end{equation*}
while the channel dimension increases correspondingly to enhance representational capacity. Specifically, each stage $s$ contains $L_s$ residual blocks with channel dimension $C_s$, where $C_1$ is defined as the base channel dimension and the subsequent stages follow a fixed scaling rule:
\begin{equation*}
    C_s = 2^{s-1} C_1, \quad s = 1, \dots, S.
\end{equation*}
For example, in the standard Tiny configuration, the channel widths are $C_s \in \{96, 192, 384, 768\}$.

Within each stage, given the input feature map $X_s \in \mathbb{R}^{H_s \times W_s \times C_s}$, the depthwise convolution is represented by a spatial weight kernel $W^{\text{dw}}_s \in \mathbb{R}^{C_s \times 1 \times k \times k}$, yielding
\begin{equation}
    Z_{1,s} = X_s \ast W^{\text{dw}}_s,
\end{equation}
where $k = 7$ corresponds to the large kernel setting used in ConvNeXt-v2 (denoted as \textit{d7}), and $\ast$ denotes the convolution operator applied independently over each channel (i.e., depthwise convolution).

This is followed by Layer Normalization and a pair of $1 \times 1$ pointwise convolutions parameterized by linear weight matrices $W_{1,s}$ and $W_{2,s}$:
\begin{equation}
    Z_{2,s} = \mathrm{GELU}\big(\mathrm{LN}(Z_{1,s}) W_{1,s} + b_{1,s}\big) W_{2,s} + b_{2,s},
\end{equation}
where $b_{1,s}$ and $b_{2,s}$ are bias.
The MLP serves as a channel-mixing module with an expansion ratio of 4, i.e.,
\begin{equation*}
    W_{1,s} \in \mathbb{R}^{C_s \times 4C_s}, \quad 
    W_{2,s} \in \mathbb{R}^{4C_s \times C_s}.
\end{equation*}

Compared to conventional Convolution-based architectures, ConvNeXt-v2 incorporates several design principles inspired by modern architectural designs, including \textbf{\textit{Large-kernel Depthwise Convolutions}} and \textbf{\textit{Inverted Bottleneck}} structures.
These design choices enhance both optimization stability and scalability, making ConvNeXt-v2 a suitable backbone for extending constraint-based pre-training beyond Transformer-based models.

\subsubsection{Constraint-based Pre-training Paradigm on Convolution-based Architectures}
To extend constraint-based pre-training to ConvNeXt-v2, we first reparameterize all depthwise convolution kernels across different stages into a unified representation. Specifically, for each stage $s$, the depthwise convolution weight is denoted as $W^{\text{dw}}_s \in \mathbb{R}^{C_s \times 1 \times k \times k}$. We then aggregate all such kernels along the stage and channel dimensions to construct a unified weight matrix:
\begin{equation}
    \mathcal{W}_\star = \mathrm{concat}\big(W^{\text{dw}}_1, W^{\text{dw}}_2, \dots, W^{\text{dw}}_S\big),
\end{equation}
where $\mathcal{W}_\star \in \mathbb{R}^{P \times (C_1 \times k \times k)}$, and $P = \sum_{s=1}^{S} C_s L_s$ denotes the total number of depthwise convolutional filters across all stages, with $L_s$ being the number of blocks in stage $s$.

The resulting unified weight matrix $\mathcal{W}_\star$ is constrained by the Kronecker-based formulation adopted in WeiT following Eq.~\eqref{equ:kro}, enabling its reconstruction from shared weight templates and size-specific weight scalers via concatenation and weighted aggregation:
\begin{equation*}
    \mathcal{W}_\star \;=\; \mathcal{T} \otimes \mathcal{S}_\star \;=\; \sum_{i=1}^{N} \mathcal{T}_i \otimes \mathcal{S}_{\star, i},
\end{equation*}
where $\mathcal{T} = \{\mathcal{T}_i\}_{i=1}^{N}$ denotes size-agnostic weight templates capturing shared convolutional patterns across stages, with each $\mathcal{T}_i \in \mathbb{R}^{1 \times (r_1\cdot r_2)}$. $\mathcal{S}_\star = \{\mathcal{S}_{\star, i}\}_{i=1}^{N}$ represents lightweight weight scalers that adapt the templates to different stage-wise and channel-wise configurations, with each $\mathcal{S}_{\star, i} \in \mathbb{R}^{P \times \frac{k^2 C_1}{r_1\cdot r_2}}$. 

To preserve the structural integrity of convolutional kernels while maintaining full spatial and channel-wise expressivity, we set $r_1 = C_1$ and $r_2 = k^2$, resulting in a representation that retains the complete degrees of freedom of convolutional filters.
Similarly, the $1 \times 1$ pointwise convolutions, which are functionally equivalent to the linear feed-forward layers in Transformer-based architectures, are modeled under the same template-based formulation. 

This unified parameterization enables a seamless extension of the Constraint-based Pre-training Paradigm to Convolution-based architectures, such as ConvNeXt-v2.

\section{Additional Training Details}
\subsection{Details of Downstream Datasets}
\label{app:dataset}
\subsubsection{Downstream Datasets in Image Classification}
\label{app:dataset_cls}
Additional downstream datasets for \textsc{Image Classification} include Oxford Flowers~\cite{nilsback2008automated}, CUB-200-2011~\cite{wah2011caltech}, Stanford Cars~\cite{gebru2017fine}, CIFAR-10~\cite{krizhevsky09}, CIFAR-100~\cite{krizhevsky09}, Food-101~\cite{bossard2014food}, and iNaturalist-2019~\cite{tan2019herbarium}. 
Table~\ref{tab:datasets_cls} summarizes the details of these seven datasets, organized in ascending order of dataset size.
\begin{table}[h]
    \centering
    \setlength{\tabcolsep}{4.5 mm}
    \caption{Details of Downstream Image Classification Datasets.}
    \vspace{-0.1in}
    \resizebox{0.7\linewidth}{!}{
        \begin{tabular}{@{}lcccc@{}}
        \toprule[0.8pt]
        \textbf{Dataset} & \textbf{Classes} & \textbf{Total} & \textbf{Training} & \textbf{Testing} \\
        \midrule[0.8pt]
        Oxford Flowers~\cite{nilsback2008automated} & 102 & 8,189 & 2,040  & 6,149\\
        CUB-200-2011~\cite{wah2011caltech} & 200 & 11,788 & 5,994 & 5,794 \\
        Stanford Cars~\cite{gebru2017fine} & 196 & 16,185 & 8,144 & 8,041\\
        CIFAR10~\cite{krizhevsky09} & 10 & 60,000 & 50,000 & 10,000 \\
        CIFAR100~\cite{krizhevsky09} & 100 & 60,000 & 50,000 & 10,000 \\
        Food101~\cite{bossard2014food} & 101 & 101,000 & 75,750 & 25,250\\
        iNaturalist-2019~\cite{tan2019herbarium} & 1010 & 268,243 & 265,213 & 3,030\\
        \bottomrule[0.8pt]
        \end{tabular}
        }
    \label{tab:datasets_cls}
\end{table}

\subsubsection{Downstream Datasets in Image Generation}
\label{app:dataset_gen}
Additional downstream datasets for \textsc{Image Generation} include CelebA-HQ~\cite{huang2018introvae}, LSUN-Bedroom~\cite{wang2017knowledge}, LSUN-Church~\cite{wang2017knowledge}, Hubble, MRI, and Pokemon. 
LSUN-Bedroom and LSUN-Church are subsets of the Large-Scale Scene Understanding (LSUN) dataset~\cite{wang2017knowledge}, consisting of scene images of bedrooms and churches, respectively, at a resolution of $256 \times 256$ pixels. 
CelebA-HQ is a high-quality variant of the CelebA dataset~\cite{liu2018large}, containing high-resolution facial images of celebrities, also resized to $256 \times 256$ pixels. Table~\ref{tab:datasets_gen} provides an overview of these six downstream datasets.

\begin{table}[h]
    \centering
    \setlength{\tabcolsep}{4.5 mm}
    \caption{Details of Downstream Image Generation Datasets.}
    \vspace{-0.1in}
    \resizebox{0.5\linewidth}{!}{
        \begin{tabular}{@{}lcc@{}}
        \toprule[0.8pt]
        \textbf{Dataset} & \textbf{Total} & \textbf{Resolution} \\
        \cmidrule[0.8pt]{1-3}
        CelebA~\cite{huang2018introvae} & 30,000 & 256$\times$256  \\
        LSUN-Bedroom~\cite{wang2017knowledge} & 3,033,042 & 256$\times$256 \\
        LSUN-Church~\cite{wang2017knowledge} & 126,227 & 256$\times$256 \\
        Hubble & 2706 & 256$\times$256  \\
        MRI & 3753 & 256$\times$256 \\
        Pokemon & 833 & 256$\times$256 \\
        \bottomrule[0.8pt]
        \end{tabular}
        }
    \label{tab:datasets_gen}
\end{table}

\subsubsection{Downstream Datasets in Embodied Control}
\label{app:rl_reward}
We conduct the main experiments for \textsc{Embodied Control} on the Flat Terrain, while additional downstream tasks include Variable Terrain, Incline, Obstacle, and Patrol. 
All tasks are constructed in the MuJoCo physics simulator~\cite{todorov2012mujoco} and are designed to comprehensively evaluate the agent's capabilities across multiple dimensions.

These tasks challenge the agent's agility, stability, and manipulation skills, by varying observations, objectives, and environmental interactions.
Collectively, these tasks constitute a diverse and challenging benchmark that systematically evaluates both the generalization and adaptability of the learned policies, following the experimental setup and task descriptions introduced by \cite{gupta2021embodied, guptametamorph, feng2025knowledge}.

\paragraph{Flat Terrain.} 
The agent is initialized on one end of a $150 \times 150$ m$^2$ flat arena.
The task requires the agent to learn stable locomotion and progress consistently along a predefined forward direction throughout an episode.
At each timestep, the agent receives a reward:
\[
    r_t = \mu \, v_x, \quad \mu = 1
\]
where $v_x$ denotes the velocity component along the $+x$-axis, corresponding to the target direction of movement.

\paragraph{Variable Terrain.} 
Similar to Flat Terrain, the agent aims to maximize forward displacement over an episode. At the start of each episode, the agent is initialized at one end of a $100 \times 100$ m$^2$ arena. 
A new terrain is generated in each episode by randomly sampling a sequence of obstacles and interleaving them with flat segments. The flat segments have lengths $l \in [1, 3]$ m along the desired direction of motion, while obstacle segments have lengths $l \in [4, 8]$ m. Three types of obstacles are considered: 
\begin{itemize}
    \item \textbf{Hills}: Parameterized by the amplitude $a$ of a $\mathrm{sin}$ wave, where $a \in [0.6, 1.2]$ m.
    \item \textbf{Steps}: A sequence of 8 steps of height $0.2$ m. Each step has equal length, with 4 steps ascending followed by 4 steps descending.
    \item \textbf{Rubble}: Random bumps generated by clipping a repeating triangular sawtooth wave at the top, with bump heights $h \in [0.2, 0.3]$ m.
\end{itemize}
The reward function is similar to Flat Terrain.

\paragraph{Incline.} The agent is tasked with maximizing forward displacement on a rectangular arena of size $150 \times 40$ m$^2$, inclined at $10^\circ$. The reward function is similar to Flat Terrain.

\paragraph{Obstacle.} The agent need to traverse a dense area of static obstacles and reach the end of a rectangular flat arena of size $150 \times 60$ m$^2$. Each obstacle has a base and height ranging from 0.5 m to 3 m, with 50 obstacles randomly initialized at the start of each episode. 
The obstacle information is provided as a terrain height map. The reward function is similar to Flat Terrain.

\paragraph{Patrol.} The agent is tasked with running back and forth between two goal locations separated by 10 m along the $x$-axis. Success in this task requires the ability to move quickly over short distances and to change direction repeatedly. At each time step, the agent receives a reward:
\[
r_t = \mu \, \Delta d_{\text{goal}}, \quad \mu = 100
\]
where $\Delta d_{\text{goal}}$ is the change in geodesic distance to the current goal between consecutive time steps, and $a$ is the action taken by the agent. Additionally, if the agent reaches within 0.5 m of a goal, the goal location is flipped and the agent receives a sparse reward of $10$.

\subsection{Details of Training Hyperparameters}
Table~\ref{tab:hyper_main} summarizes the key hyperparameters used for constraint-based pre-training across the three tasks, including \textsc{Image Classification}, \textsc{Image Generation}, and \textsc{Embodied Control}.

\begin{table}[t]
    \centering
    \setlength{\tabcolsep}{3 mm}
    \caption{Hyperparameters for WeiT under constraint-based pre-training on Image Classification, Image Generation and Embodied Control.}
    \resizebox{\linewidth}{!}{
        \begin{tabular}{@{}lrrr@{}}
        \toprule[0.8pt]
        \textbf{Hyperparameter} & \textbf{\textsc{Image Classification}} & \textbf{\textsc{Image Generation}} & \textbf{\textsc{Embodied Control}} \\
        \midrule[0.8pt]
        Number of Weight Template & 72 & 108 & 80 \\
        Shape of Weight Template & 768 $\times$ 768 & 1024 $\times$ 1024 & 256 $\times$ 256 \\
        Number of layers & 12 & 12 & 6 \\
        Number of attention heads & 12 & 16 & 4 \\
        Embedding dimension & 768 & 1024 & 256 \\
        Feedforward dimension & 3072 & 4096 & 2048 \\
        Non linearity function & GELU & GELU & ReLU \\
        \midrule
        Optimizer & AdamW & AdamW & Adam \\
        Base learning rate & 5e-4 & 1e-4 & 1e-3\\
        Warmup learning rate & 1e-6 & --- & ---\\
        Weight decay & 0.05 & 0 & 0 \\
        Optimizer momentum & 0.9 & 0.9 & --- \\
        Batch size & 1024 & 64 & 5120 \\
        Training epoch / step & 300 & 600K & 30\\
        Scheduler & Cosine Decay & --- & Cosine Decay \\
        Warmup epoch & 5 & --- & --- \\
        Drop path & 0.5 & 0.5 & 0.9 \\
        Vae & --- & stabilityai / sd-vae-ft-ema & ---\\
        Class dropout & --- & 0.1 & --- \\
        \bottomrule[0.8pt]
        \end{tabular}
    }
    \label{tab:hyper_main}
    \vspace{-0.05in}
\end{table}

\section{Additional Experimental Results}
\subsection{Zero-shot Performance after Initialization}
We evaluate the zero-shot performance of models immediately after initialization, without any task-specific fine-tuning. 
As shown in Fig.~\ref{fig:app_zero}, in \textsc{Embodied Control}, the WeiT-initialized agent with training morphologies achieves substantially higher rewards. Notably, it consistently outperforms distillation-based methods such as HyperDistill~\cite{xiong2024distilling}, suggesting that knowledge transfer via weight templates induces more structured and transferable representations than distillation-based approaches, thereby yielding improved generalization and stronger zero-shot performance.

We further assess initialization quality in \textsc{Image Classification}. As shown in Fig.~\ref{fig:app_wave_imagenet}, WeiT consistently achieves superior initial classification performance compared to other methods, such as TLEG~\cite{xia2024transformer}, which imposes simple structural priors via layer-wise sharing. This finding is consistent with the observations above, further suggesting that the learned weight templates effectively capture transferable knowledge across diverse model scales.

\begin{figure*}[h]
    \centering
    \includegraphics[width=\linewidth]{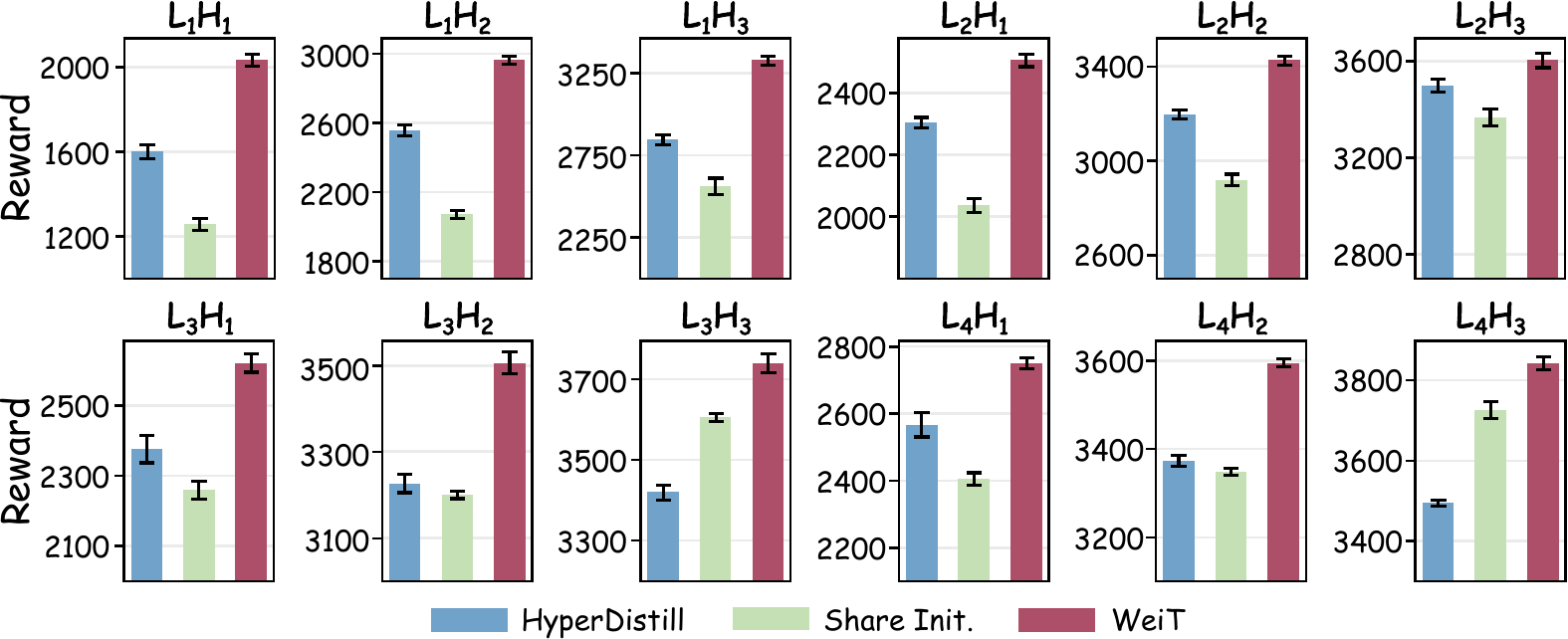}
    \vspace{-0.25in}
    \caption{Zero-shot Initialization Performance across Training Morphologies.}
    \label{fig:app_zero}
\end{figure*}

\begin{figure*}[t]
  \centering
  \includegraphics[width=\linewidth]{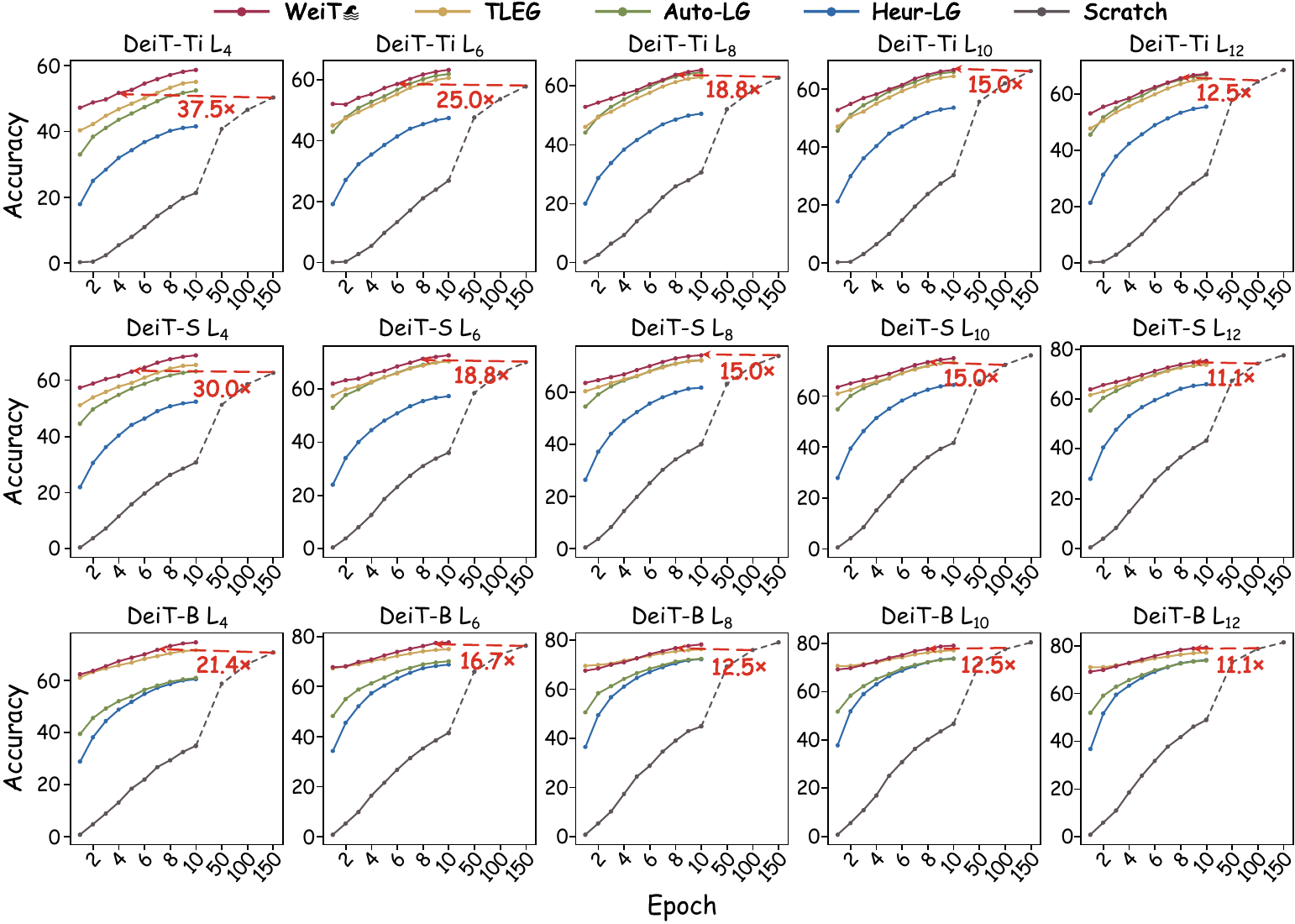}
  \vspace{-0.25in}
  \caption{Training dynamics of \textsc{Image Classification} on ImageNet-1K. We report detailed optimization trajectories, where scalable initialization methods (e.g., WeiT) are trained for 10 epochs (corresponding to Table~\ref{tab:init_class}) and compared with models trained from scratch for 150 epochs. 
  Note: This figure is directly adapted from the Appendix of WeiT\waveicon~\cite{feng2024wave}.
  }
  \label{fig:app_wave_imagenet}
\end{figure*}

\subsection{Empirical Analysis of Learning Efficiency}
To more intuitively demonstrate the effectiveness of the proposed method, we visualize the training dynamics on \textsc{Image Classification}. 
Fig.~\ref{fig:app_wave_imagenet} reports the classification accuracy curves of scalable initialization methods (trained for 10 epochs) and models trained from scratch (trained for 150 epochs).

Overall, WeiT\waveicon consistently outperforms other comparable methods, including Heur-LG~\cite{wang2022learngene}, Auto-LG~\cite{wang2023learngene}, and TLEG~\cite{xia2024transformer}, while also significantly improving training efficiency. 
In particular, compared to models trained from scratch, WeiT\waveicon-initialized models achieve competitive performance within only a few epochs of training, and in certain settings can match the accuracy of 150-epoch training from scratch after just one epoch.
Taking DeiT-B with 12 layers (i.e., DeiT-B $L_{12}$) as an example, WeiT\waveicon reduces the training cost by approximately $11.1\times$ compared to training from scratch. 
Notably, this efficiency gain becomes even more pronounced for smaller models, reaching up to $37.5\times$ in DeiT-Ti ($L_4$). These results further demonstrate the strong learning capability and superior training efficiency induced by the proposed initialization.

\begin{figure*}[t]
  \centering
  \includegraphics[width=\linewidth]{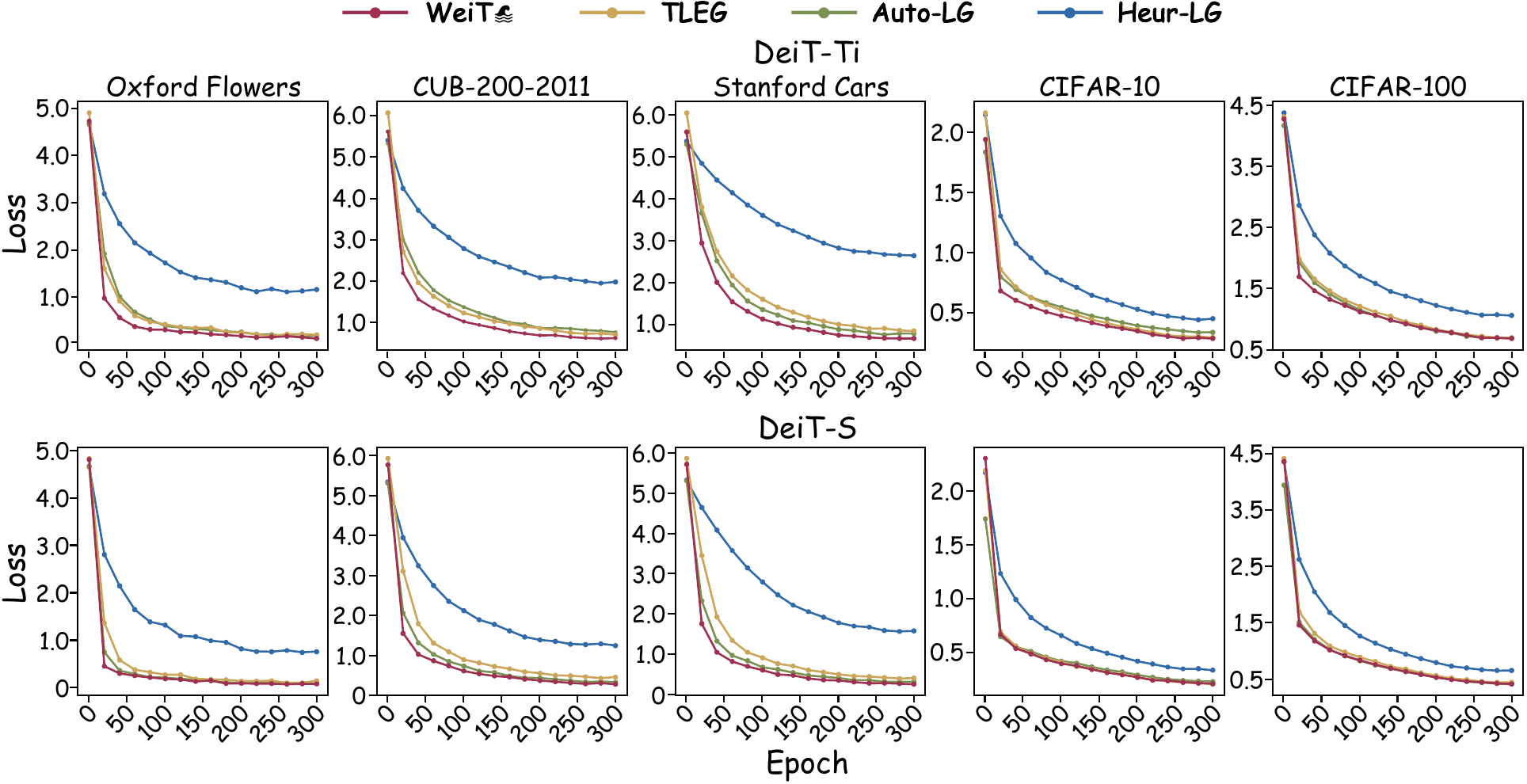}
  \caption{Training dynamics of \textsc{Image Classification} on small and medium-scale downstream datasets, where we report detailed loss trajectories corresponding to Table~\ref{tab:downstream_class}. Note: This figure is directly adapted from the Appendix of WeiT\waveicon~\cite{feng2024wave}.}
  \label{fig:app_wave_downstream}
\end{figure*}

Such strong learning ability is also reflected in WeiT\waveicon-initialized models on downstream datasets. We further visualize the training loss trajectories on small- and medium-scale datasets, including Oxford Flowers~\cite{nilsback2008automated}, CUB-200-2011~\cite{wah2011caltech}, Stanford Cars~\cite{gebru2017fine}, CIFAR-10~\cite{krizhevsky09}, and CIFAR-100~\cite{krizhevsky09}. 
As shown in Fig.~\ref{fig:app_wave_downstream}, WeiT\waveicon-initialized models exhibit consistently faster loss convergence, indicating improved optimization efficiency and enhanced learning capability in downstream tasks.

\subsection{Effect of Weight Templates on Different Components}
We further analyze the impact of weight templates on the initialization of different model components in \textsc{Image Classification}, as detailed in Table~\ref{tab:ablation}.
The results show that most components can be effectively initialized using the structured knowledge encapsulated in the weight templates. 

In contrast, parameters associated with normalization layers and bias terms are more data-dependent and typically involve fewer parameters, making them easier to learn directly from data; thus, applying weight templates to these components is less critical. 
Notably, the attention mechanism—comprising the Query, Key, and Value projections—exhibits a stronger dependence on structured knowledge, underscoring its critical role in effective initialization.

\begin{table}[h]
    \centering
    \setlength{\tabcolsep}{2.5 mm}
    \caption{Ablation study on weight templates initializing different components of DeiT, with models consisting of 6 layers.}
    \vspace{-0.05in}
    \resizebox{0.68\linewidth}{!}{
        \begin{tabular}{@{}lccccccc@{}}
        \toprule[0.8pt]
        Methods & MSA-QKV & MSA-Proj. & FFN & Norm & Ti & S & B\\
        \midrule[0.8pt]
        He Init. & & & & & 40.6 & 49.4 & 53.1 \\
        \midrule
        \multirow{5}{*}{WeiT\waveicon} & & $\checkmark$ & &  & 50.2 & 57.9 & 60.4 \\
                                & $\checkmark$ & & & & 52.3 & 60.3 & 64.6 \\    
                                & & &$\checkmark$&  & 58.7 & 66.8 & 69.0 \\
                                & $\checkmark$ & $\checkmark$ & $\checkmark$ &  & 63.1 & 72.6 & 77.4 \\
                                & \cellcolor{blue!12}{$\checkmark$} & \cellcolor{blue!12}{$\checkmark$} & \cellcolor{blue!12}{$\checkmark$} & \cellcolor{blue!12}{$\checkmark$} & \cellcolor{blue!12}{\textbf{63.2}} & \cellcolor{blue!12}{\textbf{72.7}} & \cellcolor{blue!12}{\textbf{77.5}} \\
        \bottomrule[0.8pt]
        \end{tabular}
        }
    \label{tab:ablation}
    \vspace{-0.05in}
\end{table}

\subsection{Effect of Number and Shape of Weight Templates}
We analyze the influence of the number and size of weight templates on initializing downstream models. 
As shown in Table~\ref{tab:analysis}, reducing the number of templates leads to a significant performance degradation, suggesting that an insufficient number of parameters limits the capacity to capture comprehensive size-agnostic knowledge.

Furthermore, we investigate reducing the size of weight templates to increase flexibility for width expansion. 
However, such flexibility sacrifices the performance for the insufficient structured knowledge in weight templates.
In contrast, by aligning the base dimensionality of weight templates with the \textit{embedding dimensions} of the pre-trained model, WeiT~\waveicon facilitates the effective capture of structured knowledge, as illustrated in Fig.~\ref{fig:structure}a.

\begin{table}[h]
    \centering
    \setlength{\tabcolsep}{2 mm}
    \caption{Analysis of the Number and Shape of Weight Templates.}
    \vspace{-0.1in}
    \resizebox{0.7\linewidth}{!}{
        \begin{tabular}{@{}lccc||ccc||ccc@{}}
        \toprule[0.8pt]
        & \multicolumn{3}{c}{$L_6H_3$} & \multicolumn{3}{c}{$L_6H_6$} & \multicolumn{3}{c}{$L_6H_{12}$}\\
        \cmidrule(lr){2-4}
        \cmidrule(lr){5-7}
        \cmidrule(l){8-10}
        & Param. & Shape & Acc. & Param. & Shape & Acc. & Param. & Shape & Acc. \\
        \cmidrule[0.8pt]{1-10}
        $\downarrow$ Num. & 0.9 & 192$^2$ & 57.5 & 
        2.6 & 384$^2$ & 68.4 & 
        8.6 & 768$^2$ & 75.7 \\        
        $\downarrow$ Shape & 1.3 & 96$^2$ & 60.3 & 
        4.4 & 192$^2$ & 71.6 & 
        15.8 & 384$^2$ & 75.8 \\
        \midrule
        \cellcolor{blue!12}{WeiT\waveicon} & \cellcolor{blue!12}{1.3} & \cellcolor{blue!12}{192$^2$} & \cellcolor{blue!12}{\textbf{63.2}} & 
        \cellcolor{blue!12}{4.4} & \cellcolor{blue!12}{384$^2$} & \cellcolor{blue!12}{\textbf{72.7}} & 
        \cellcolor{blue!12}{15.8} & \cellcolor{blue!12}{768$^2$} & \cellcolor{blue!12}{\textbf{77.5}} \\
        \bottomrule[0.8pt]
        \end{tabular}
        }
    \label{tab:analysis}
    \vspace{-0.1in}    
\end{table}

\subsection{Supplementary Analysis on Embodied Control}
Due to space constraints in the main paper, Table~\ref{tab:larger}, Table~\ref{tab:constait}, and Table~\ref{tab:abl} report only the final rewards for \textsc{Embodied Control}. 
In this section, we provide the full training dynamics in Fig.~\ref{fig:app_large}, Fig.~\ref{fig:app_rule}, and Fig.~\ref{fig:app_wodrop}, offering a more comprehensive view of the learning behavior and performance evolution across different settings.

\begin{figure*}[t]
    \centering
    \includegraphics[width=\linewidth]{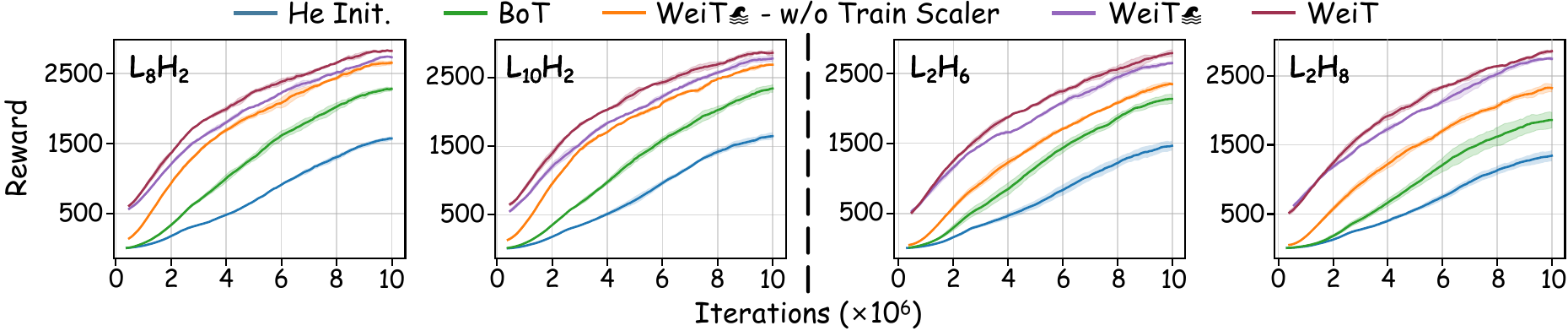}
    \vspace{-0.25in}
    \caption{Supplementary results on the performance of scale-up initialization for larger models, providing detailed evaluation on \textsc{Embodied Control} to complement Table~\ref{tab:larger}.}
    \label{fig:app_large}
\end{figure*}

\begin{figure*}[h]
    \centering
    \includegraphics[width=\linewidth]{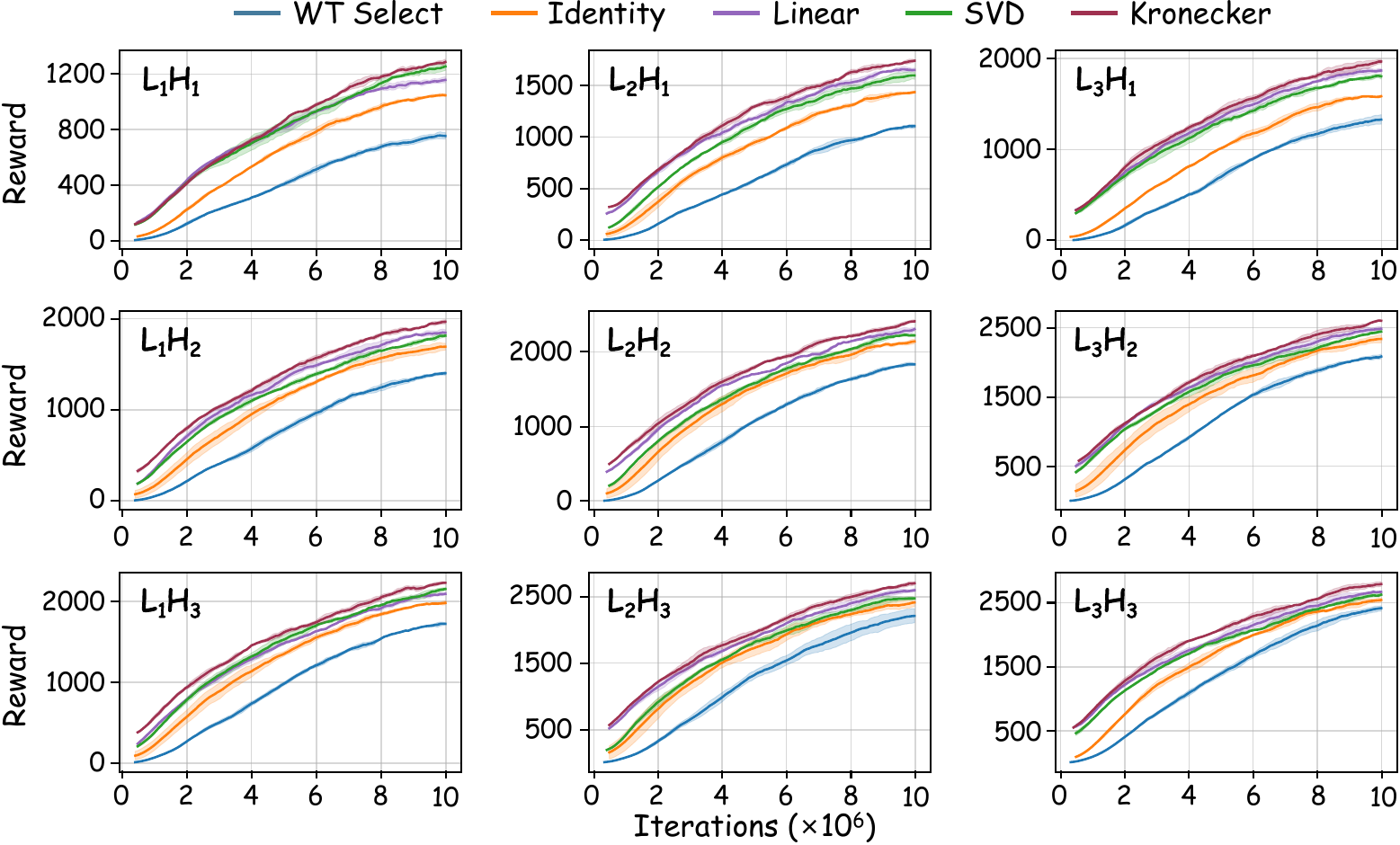}
    \vspace{-0.25in}
    \caption{Supplementary ablation results on constraint types for variable-sized model initialization, providing detailed evaluation on \textsc{Embodied Control} to complement Table~\ref{tab:constait}.}
    \label{fig:app_rule}
\end{figure*}

\begin{figure*}[h]
    \centering
    \includegraphics[width=\linewidth]{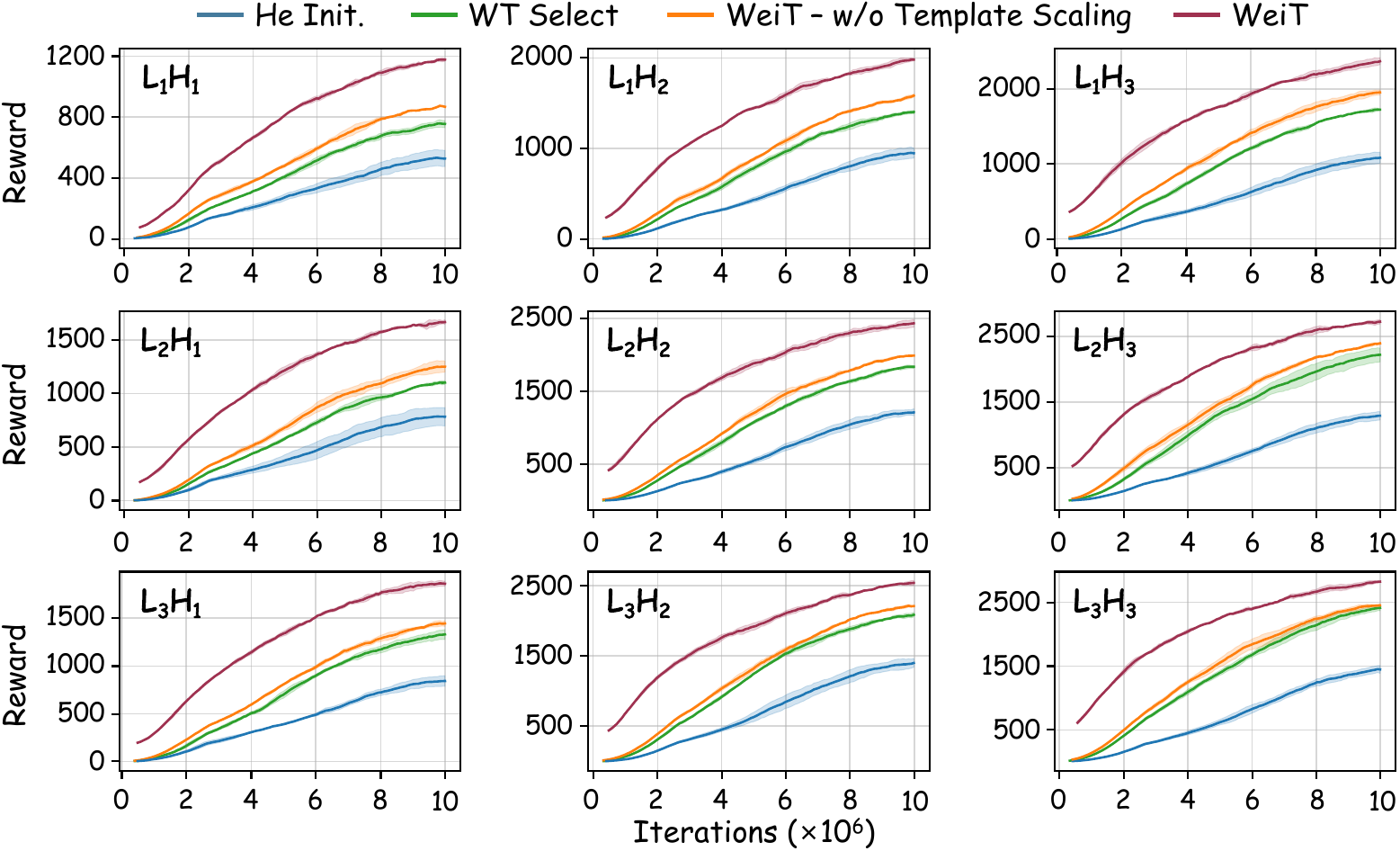}
    \vspace{-0.25in}
    \caption{Supplementary results on the performance of scale-up initialization for larger models, providing detailed evaluation on \textsc{Embodied Control} to complement Table~\ref{tab:larger}.}
    \label{fig:app_wodrop}
\end{figure*}

\section{Limitations and Future Work}
\label{sec:limitations_and_future_work}
While the proposed Constraint-based Pre-training paradigm, instantiated by WeiT, provides a scalable and effective initialization framework, it still has several limitations that suggest promising directions for future work.

\subsubsection{Limitations} 
First, the current Kronecker-based constraints in WeiT are primarily designed for standard dense operators, such as linear projections and regular convolutions. Extending them to more heterogeneous or structured operations, e.g., depth-wise convolutions, requires careful redesign to preserve their specific inductive biases.

Second, the parameter-efficient initialization relies on frozen weight templates $\mathcal{T}$, which improves stability and regularization but assumes sufficient cross-domain transferability. Under large domain shifts, e.g., from natural images to substantially different modalities such as medical images, frozen weight templates may limit the adaptability of the learned scalers $\mathcal{S}$. In such cases, relaxing the freezing of $\mathcal{T}$ may be beneficial to improve flexibility.

\subsubsection{Future Work}
Based on these observations, we highlight two directions for future research:
\begin{itemize}
    \item \textit{Cross-Modal Templates.} Extending WeiT to learn shared templates across modalities, e.g., vision and text, is a promising direction. In this setting, only lightweight modality-specific scalers need to be adapted, leading to more general and parameter-efficient models.

    \item \textit{Adaptive Structure Learning.} Instead of fixing the Kronecker-based constraints and template configurations, future work could explore learning these components in a data-driven manner, e.g., via differentiable structure search, to better adapt the factorization to different tasks and architectures.
\end{itemize}

\newpage

\vfill

\end{document}